\documentclass[final]{article}

\usepackage[utf8]{inputenc}
\usepackage[T1]{fontenc}
\usepackage{kr}

\usepackage[center]{caption}
\usepackage[dvipsnames]{xcolor}
\usepackage{amsmath, amsfonts, amsthm, amssymb}
\usepackage{thmtools,xspace}

\usepackage[top=1in, bottom=1in, left=1in, right=1in]{geometry}
\usepackage[alphabetic,lite,backrefs]{amsrefs} 
\usepackage[backref=page, colorlinks=true, allcolors=blue]{hyperref}
\usepackage{cleveref}
\usepackage{comment}
\usepackage{tikz}
\usepackage{placeins}
\usepackage{multicol, wrapfig}
\usepackage{ulem}
\usetikzlibrary{graphs}
\usepackage{changepage}
\usepackage{mathrsfs}
\usepackage{tikz-cd}
\usepackage{mathtools} 
\usepackage{bm} 
\usepackage{comment} 
 \usetikzlibrary{positioning, shapes.geometric, arrows.meta, quotes}
\usepackage[shortlabels]{enumitem} 

\usepackage[colorinlistoftodos,prependcaption,textsize=scriptsize, obeyFinal]{todonotes}

\usepackage{dblfloatfix}
\usepackage{pifont}

\usepackage{booktabs}

\tikzstyle{arg}=[
    circle,
    minimum size =5mm,
    draw=black,
    thick,
    style={font=\small},
]
\tikzstyle{arg2}=[
    circle,
    minimum size =2mm,
    draw=black,
    thick,
    style={font=\footnotesize},
]

\tikzstyle{dot}=[
    circle,
    fill,
    minimum size =4mm,
    on grid,
]
\tikzstyle{dot2}=[
    circle,
    fill,
    style={font=\tiny},
    on grid,
    minimum size=0.1mm,
]
\tikzstyle{lab}=[
    style={font=\normalsize},
]

\tikzstyle{oval}=[
    ellipse, 
    text width = 2cm,
    minimum height=1cm,
    draw,
    align = center,
]

\tikzstyle{ov}=[    
    ellipse, 
    draw,
    align = center,
]

\newcommand\ForallWidget[4]{
    \node (center#1) [right of =#2]{};
    \node[ov] (x#1) [above =7.5mm of center#1] {$x_{#1}$};
    \node[ov] (notx#1) [below =7.5mm of center#1] {$\neg x_{#1}$};
    \node[lab] (allx) [above = 7.5mm of x#1] {$\forall x_{#1}$};

    \draw[->] (x#1) to (#3);
    \draw[->] (x#1) to (#4);
    \draw[->] (notx#1) to (#3);
    \draw[->] (notx#1) to (#4);
    \draw[<->] (x#1) to (notx#1);
}

\newcommand\ExistsWidget[4]{
    \node (center#1) [right  of =#2]{};
    \node[ov] (x#1) [above =7.5mm of center#1] {$x_{#1}$};
    \node[ov] (notx#1) [below =7.5mm of center#1] {$\neg x_{#1}$};
    \node[lab] (allx) [above = 7.5mm of x#1] {$\exists x_{#1}$};

    \draw[->] (x#1) to (#3);
    \draw[->] (x#1) to (#4);
    \draw[->] (notx#1) to (#3);
    \draw[->] (notx#1) to (#4);
    \draw[<->] (x#1) to (notx#1);
}

\newcommand\SwitchWidget[3]{
    \node[dot] (switch) [right of = #1]{};

    \draw[->] (switch) to (#2);
    \draw[->] (switch) to (#3);
}
\newcommand\SwitchWidgetWithLabel[4]{
    \node[arg, fill= gray!30, shape=rectangle, rounded corners ] (switch) [right of = #1]{$b_{#4}$};

    \draw[->] (switch) to (#2);
    \draw[->] (switch) to (#3);
}

\newcommand\PhiWidget[3]{
    \node[arg] (phi) [right of = #1]{$\phi$};

    \draw[->] (phi) to (#2);
    \draw[->] (phi) to (#3);
}

\makeatletter
\providecommand\@dotsep{5}
\renewcommand{\listoftodos}[1][\@todonotes@todolistname]{%
  \@starttoc{tdo}{#1}}
\makeatother

\theoremstyle{plain}
	\newtheorem{theorem}{Theorem}[section]
	\newtheorem{lemma}[theorem]{Lemma}
    \newtheorem{corollary}[theorem]{Corollary}
    \newtheorem{proposition}[theorem]{Proposition}

\theoremstyle{definition}
    \newtheorem{defn}[theorem]{Definition}
    \newtheorem*{defn*}{Definition}
    \newtheorem{example}[theorem]{Example}
    \newtheorem*{example*}{Example}
    
\theoremstyle{remark}
	
	\newtheorem*{remark*}{Remark}

    \newtheorem*{claim*}{Claim}

\usepackage{mathtools}

\DeclarePairedDelimiterX\Set[1]\{\}{#1}

\AtBeginEnvironment{example}{%
  \pushQED{\qed}%
}
\AtEndEnvironment{example}{\popQED\endexample}

\newcommand{\mc}[1]{\mathcal{#1}}

\def\wadm{wad_{bbu}}

\renewcommand{\S}{\mathcal{S}}

\def\sl{^{sim}}
\def\th{^{th}}

\def\L{{\mc{L}}}

\def\U{{\mathcal U}}
\def\F{{\mathcal F}}
\def\G{{\mathcal G}}

\def\att{\rightarrowtail}

\def\I{Pro}
\def\II{Opp}
\def\dom{\mathrm{dom}}

\def\tin{\texttt{in}\xspace}
\def\tout{\texttt{out}\xspace}
\def\tun{\texttt{undec}\xspace}

\def\ten{ten}
\def\strten{ten^{str}}
\def\staten{ten_{sta}}
\DeclareMathOperator{\cred}{cred}

\setlist[enumerate]{leftmargin=0.6cm}

\newcommand{\repeatcaption}[2]{%
  \renewcommand{\thefigure}{\ref{#1}}%
  \captionsetup{list=no}%
  \caption{#2 (repeated from page \pageref{#1})}%
  \addtocounter{figure}{-1}
}

\def\Args{F}
\def\Att{\att}
\def\AF{(\Args,\Att)}
\def\cog{\text{cog}}
\def\cf{\text{cf}}

\def\cycog{\text{cycog}}

\title{Tenability and Weak Semantics: Modeling Non-uniform Defense -- Extended Version}

\author{%
Uri Andrews $^1$\and
Luca San Mauro $^2$\and
John Spoerl $^1$\\
\affiliations
$^1$Department of Mathematics, University of Wisconsin--Madison \\
$^2$Dipartimento di Ricerca e Innovazione Umanistica,
University of Bari, Italy\\
\emails
andrews@math.wisc.edu,
luca.sanmauro@gmail.com,
john.spoerl@wisc.edu
}

\begin{document}

\maketitle

\begin{abstract}

In Dung-style abstract argumentation, various semantics capture notions of acceptability of arguments. The admissibility semantics capture the notion that an argument can be consistently defended from any potential counterargument.
Weak semantics often relax the demands of admissibility by restricting which counterarguments must be taken seriously (e.g., discounting self-defeating or otherwise incoherent attacks). Many prominent proposals for weak semantics remain extension-based in a stronger sense. While these semantics discount attacks from arguments which are considered unreasonable, they still require a uniform defense against all reasonable arguments, even if they are collectively inconsistent.
This uniformity can be too demanding when defensibility is inherently strategic, and thus the appropriate reply depends on the opponent’s line of attack.

We introduce \textit{tenability}, a family of dialogue-based semantics that formalize when a designated argument (or a set of arguments) can be maintained in debate by a proponent against any conflict-free attack which the opponent may present. The approach is motivated by three natural benchmark patterns: self-defeating attack, floating assignment, and disjunctive reinstatement, on which
tenability behaves differently from all weak semantics previously considered in the literature.
We define three variants---static tenability, tenability, and strong tenability---via monotone commitment games over finite conflict-free moves, differing in the obligations imposed on the disputants. We establish the relative strength of these notions, prove implications and separations with previously studied weak semantics, and we analyze computational complexity on finite frameworks: deciding static tenability is $\Pi^P_2$-complete, while deciding tenability and strong tenability is $PSPACE$-complete.
\end{abstract}
    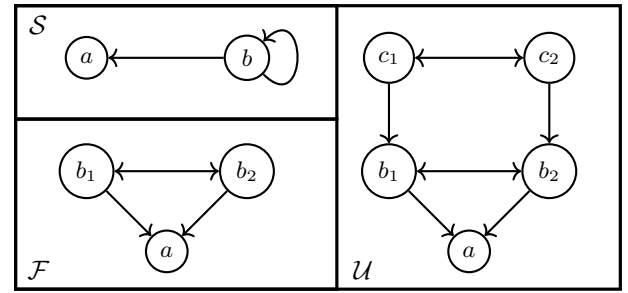
\begin{figure}[b!]\centering
        \begin{tikzpicture}[node distance =15mm, thick]

        \draw[very thick] (-2,-0.5) rectangle (6,3.25);
        \draw[very thick] (-2,-0.5) rectangle (2.25,1.75);
        \draw[very thick] (-2,1.75) rectangle (2.25,3.25);
            
            \node[arg] at (0,0) (af) {$a$}; 
            \node[arg] (bf1) [above left of=af] {$b_1$};
            \node[arg] (bf2) [above right of= af] {$b_2$};
            \draw[->] (bf1) to (af);
            \draw[->] (bf2) to (af);
            \draw[<->] (bf1) to (bf2);
            \node at (-1.7,-0.25) (F) {$\F$};
            
            \node[arg] (as)[above of = bf1] {$a$};
            \node[arg] (bs) [above of = bf2] {$b$};
            \draw[->] (bs) to (as);
            \draw[->] (bs) to [in = 45, out = 315, looseness = 5] (bs);
            \node at (-1.7, 3) (S) {$\S$};

            \node[arg] at (4,0) (au) {$a$}; 
            \node[arg] (bu1) [above left of=au] {$b_1$};
            \node[arg] (bu2) [above right of=au] {$b_2$};
            \node[arg] (cu1) [above of = bu1] {$c_1$};
            \node[arg] (cu2) [above of = bu2] {$c_2$};
            \draw[->] (bu1) to (au);
            \draw[->] (bu2) to (au);
            \draw[<->] (bu1) to (bu2);
            \draw[->] (cu1) to (bu1);
            \draw[->] (cu2) to (bu2);
            \draw[<->] (cu1) to (cu2);
            \node at (2.6,-0.25) (U) {$\U$};
        \end{tikzpicture}
        \caption{$\S$: Self-Defeating Attack\\ $\F$: Floating Assignment\\ $\U$: Disjunctive Reinstatement}
        \label{fig: main examples}
    \end{figure}

\section{Introduction}

Dung's abstract argumentation frameworks (AFs)~\cite{Dung1995} have become a core formalism for defeasible reasoning in knowledge representation~\cite{BenchCaponDunne2007,RahwanSimari2009}. An AF represents arguments as abstract entities related by an attack relation; a semantics then determines which sets of arguments (extensions) can be collectively maintained. A foundational role in this landscape is played by \textit{admissibility}-based semantics: a set should be internally consistent (conflict-free) and able to withstand counterarguments by defending each of its members against \textit{every} attacker. This uniform defense ideal underlies the standard Dung-style semantics (complete, grounded, preferred, stable) and their numerous refinements~\cite{Dung1995,BaroniCaminadaGiacomin2011}.

At the same time, it is equally well known that strict admissibility may be overly conservative in the presence of pathological attack patterns, most notably those featuring self-defeating arguments and odd cycles. In such configurations, an attack may arguably fail to constitute a serious challenge precisely because its source undermines itself, yet admissibility treats it on par with ordinary attacks~\cite{BodanzaTohme2009,BaroniGiacomin2007}. This and related concerns have led to a large and still expanding body of work proposing alternatives that weaken or reshape the admissibility requirement, yielding a landscape of semantics that is both broad and structurally diverse~\cite{BaroniGiacomin2007,BaroniGiacomin2009,BaroniCaminadaGiacomin2011,BaroniEtAlHandbook2018,vanDerTorreVesic2017}. Classic examples include non-admissible or cycle-sensitive proposals (e.g.\ stage-style or semi-stable variants)~\cite{Verheij1996,CaminadaCarnielliDunne2012} as well as resolution-based families designed to tame problematic cyclic structures while retaining key principles~\cite{BaroniDunneGiacomin2011}.

Within this broad spectrum, \textit{weak semantics} aim to preserve the defense-based intuition while refining which attacks count or how they must be responded to. For orientation, we recall three influential lines that will serve as comparison points throughout the paper.

\begin{itemize}
\item 
Cogency-based approaches start from a comparative idea: extensions are ordered by a binary \textit{at least as cogent as} relation $X\succeq Y$, which captures that $X$ defends itself against attacks from $Y$.
In the standard cogency case, this admits an equivalent, more operational reading: selecting the undominated extensions amounts to retaining a Dung-style notion of defence while discounting self-defeating attacks.~\cite{BodanzaTohme2009,BodanzaTohmeSimari2016,BlumelUlbricht}.

\item Weak admissibility weakens defense \textit{recursively} via reduct-style constructions that filter which attackers remain relevant once a candidate position is assumed~\cite{BaumannBrewkaUlbricht2020}. This yields semantics that mitigate odd-cycle effects while staying close in spirit to defense-based acceptability, and has already prompted further extensions and complexity analyses~\cite{DvorakUlbrichtWoltran2022,Dauphin2020APA}.

\item Weakly complete semantics revise Dung's complete labelings by relaxing the propagation of undecidedness: under suitable conditions, an argument may be accepted even if some attackers remain undecided, via the mechanisms of \textit{undecidedness blocking} or \textit{propagation}~\cite{DondioLongo2021}. This provides a principled way to model acceptance under unresolved counterevidence and connects naturally to ambiguity-blocking intuitions.
\end{itemize}

Despite their differences, the above approaches share a common structural commitment typical of extension and labeling-based semantics: acceptability is witnessed by a \textit{single} stance that must be adequate against \textit{all} attacks that the semantics deems relevant, \textit{simultaneously}. This ``uniform witness'' pattern is also reflected in classical proof-theoretic and dialogical characterizations of argumentation semantics~\cite{VreeswijkPrakken2000,Prakken2005,CaminadaWu2009,McBurneyParsons2009}.
However, in many debates and evidential settings, defensibility is inherently \textit{strategic}: the appropriate reply depends on which consistent line of attack an opponent commits to, and mutually incompatible challenges need not be met \textit{jointly} by one fixed, consistent position. Requiring a single stance that neutralizes every consistent challenge at once can therefore be too demanding.

This paper develops \textit{tenability}, a family of semantics designed to capture \textit{non-uniform} defendability. The key conceptual shift is from \textit{there exists one stance answering all relevant attacks} to \textit{for every consistent attack, there exists an appropriate response}.
We formalize this idea through monotone commitment protocols in which players build finite conflict-free positions by increasing moves, yielding three notions with increasing dialogical structure: \textit{static tenability}, \textit{strong tenability}, and \textit{tenability}. These notions are motivated and contrasted on a small set of benchmark patterns in Section~\ref{sect:motivation}, formally defined in Section~\ref{sect: notions of tenability}, and studied from a mathematical and computational perspective in Section~\ref{sec:math}. Section~\ref{sec:other semantics} positions tenability within the weak-semantics landscape and highlights the distinctive role of \textit{non-uniform} defense, while Section~\ref{sec: discussion} discusses limitations and directions for further work.

\section{Motivation}\label{sect:motivation}

Our starting point is a simple query: When is a designated argument $a$ defensible in a dispute?
In Dung-style abstract argumentation, defensibility is most often verified through a \textit{uniform witness}:
acceptance is certified by a single stance (an extension or a labeling) that must be adequate against
all counterarguments deemed relevant by the semantics, simultaneously.
This uniformity is also reflected in standard dialogical and game-theoretic characterisations of Dung-style semantics,
where acceptance is tied to the existence of a (winning) defense that is robust across all opponent choices \cite{VreeswijkPrakken2000,Prakken2005}.

In this paper we isolate a complementary, strategy-oriented reading: \textit{defensibility is the ability to answer any consistent line of attack}, possibly with different replies against incompatible challenges.
This motivates \textit{tenability} as a semantics of \textit{non-uniform defense}.

To make the issue concrete, we focus on three small benchmark patterns shown in Figure~\ref{fig: main examples}:
\textit{self-defeating attack} $\mathcal{S}$, \textit{floating assignment} $\F$, and \textit{disjunctive reinstatement} $\U$.
Across these benchmarks, we test whether $a$ can be maintained against reasonable challenges.
In an abstract AF, a minimal proxy for a reasonable line of challenge is legitimacy:
an opponent should not be allowed to press, as a single line of contestation, a bundle of mutually incompatible arguments.
Accordingly, we treat a legitimate attack as a conflict-free set of arguments advanced as one line of challenge.

\subsubsection*{Self-defeating attack ($\mathcal S$).}
In $\mathcal S$, the only attacker of $a$ is self-defeating.
Such patterns have long been recognised as problematic for classic admissibility and related semantics:
if an attacker undermines itself (directly or via odd-cycle phenomena), it is unclear why it should count as a decisive challenge
\cite{Dung1995,BodanzaTohme2009,BaumannBrewkaUlbricht2020,DondioLongo2021}.
Under a consistency constraint, a self-defeating argument cannot constitute a reasonable attacking move.
This aligns with the guiding intuition behind several weak-semantics proposals that discount self-defeating (or otherwise
problematic) attacks when assessing defense \cite{BodanzaTohme2009,BaumannBrewkaUlbricht2020,DondioLongo2021}.

\subsubsection*{Floating assignment ($\F$).}
The framework $\F$ exhibits a different phenomenon.
Here $a$ is attacked by two incompatible counterarguments $b_1$ and $b_2$ (they attack each other).
While $\{b_1,b_2\}$ is not a legitimate line of attack, each singleton $\{b_1\}$ and $\{b_2\}$ \textit{is} legitimate.
Yet $a$ has no satisfactory reply to either: any attempt to counterattack $b_1$ must employ $b_2$ as defender, and vice versa,
but either choice immediately destroys consistency with $a$ itself. 
So $\F$ captures the principle that a position is \textit{untenable} whenever the opponent can mount a legitimate challenge that the proponent cannot answer
even in isolation.
This benchmark is widely discussed in connection with ambiguity and standards of proof in legal-style readings,
where incompatible testimonies may or may not suffice for a verdict depending on the evidentiary threshold
(see, e.g., \cite{Horty2001}).
For tenability, the decisive point is structural: inability to answer a consistent singleton challenge already voids defensibility.

\subsubsection*{Disjunctive reinstatement ($\U$).}
The crucial benchmark for our purposes is $\U$.
As in $\F$, the designated argument $a$ faces two incompatible attackers $b_1$ and $b_2$.
Unlike $\F$, each attacker can be met by  independently consistent defenses: $c_1$ neutralises $b_1$ and $c_2$ neutralises $b_2$,
and each of $\{a,c_1\}$ and $\{a,c_2\}$ is conflict-free.
What fails is precisely uniform defense:
there is no single conflict-free set extending $\{a\}$ that counterattacks \textit{both} $b_1$ and $b_2$ at once, since $c_1$ and $c_2$ are themselves incompatible.
This illustrates the gap between a uniform-witness reading of defense and a strategic reading. Under a strategic reading, once the opponent commits to a consistent line of challenge, the proponent may tailor her reply to that line. Thus $a$ can be maintained even though it is not uniformly defended by any fixed conflict-free stance. 

\smallskip

This gap also exposes an asymmetry that can arise when uniformity is imposed at the level of witnessing stances:
the proponent is required to provide a single consistent ``story'' that addresses mutually incompatible accusations jointly,
even when no rational opponent would be entitled to maintain those accusations together as one consistent line.
We refer to this as the \textit{Master versus Ignoramus} effect:
the proponent is treated as a master who must answer every objection within one account,
while the opponent enjoys the freedom of switching between incompatible objections without committing to their joint consistency.

To make the strategic reading more intuitive, consider the following toy narrative modeled by $\U$.
\begin{quote}\small
\textit{Miss Scarlett is accused of a crime that occurred in the ballroom. The following testimonies are available as evidence.}
\end{quote}
\begin{enumerate}
    \item[$a$:] \textit{Miss Scarlett claims she was not at the scene at the time of the crime.}
    \item[$b_1$:] \textit{Colonel Mustard claims he saw Miss Scarlett leaving the west exit of the library at the time of the crime.}
    \item[$b_2$:] \textit{Mr.\ Green claims he saw Miss Scarlett leaving the east exit of the library at the time of the crime.}
    \item[$c_1$:] \textit{Mrs.\ White testifies that Colonel Mustard and Professor Plum were with her in the lounge throughout the evening.}
    \item[$c_2$:] \textit{Professor Plum testifies that Mr.\ Green and Mrs.\ White were with him in the study throughout the evening.}
\end{enumerate}

If the opponent challenges $a$ by committing to the line $\{b_1\}$, the proponent can respond with $\{c_1\}$;
if the opponent instead commits to $\{b_2\}$, the proponent can respond with $\{c_2\}$.
Crucially, in $\U$ the opponent is not entitled to press $\{b_1,b_2\}$ as a \textit{single} consistent line of attack,
since $b_1$ and $b_2$ are mutually incompatible.

\smallskip

Taken together, $\mathcal S$, $\F$, and $\U$ isolate three desiderata for a tenability-style notion of defensibility:
\begin{enumerate}
    \item[$1.$] Inconsistent challenges should not by themselves undermine defensibility ($\mathcal S$).
    \item[$2.$] Every consistent line of attack must be answerable on its own; if some consistent attack admits no consistent reply, the position is untenable ($\F$).
    \item[$3.$] The defense may depend on the opponent's consistent line of challenge; incompatible attacks need not be rebutted jointly by a single consistent stance ($\U$).
\end{enumerate}

These desiderata motivate our three tenability notions introduced next.
\textit{Static tenability} captures the one-shot strategic idea already visible in $\U$:
for every consistent challenge $B$ attacking a position $A$, there exists a conflict-free extension $C\supseteq A$ that counters every attack from $B$.
\textit{Strong tenability} and \textit{tenability} lift this to dynamic disputes via monotone commitment protocols:
the proponent must maintain and defend her accumulated commitments over time. Formal definitions are presented in Section~\ref{sect: notions of tenability}.

\section{Notions of Tenability}\label{sect: notions of tenability}

We now formalize \textit{tenability} as a family of weak semantics grounded in the strategy-based reading developed in Section~\ref{sect:motivation}:
defensibility is the ability to answer any consistent line of attack, possibly with different replies against incompatible challenges.
Technically, we capture this via monotone commitment protocols over finite conflict-free moves.

We introduce three notions: \textit{static tenability} (Section \ref{subsec: static tenability}), \textit{strong tenability} (Section \ref{subsec: strong tenability}), and \textit{tenability} (Section \ref{subsec: tenability}).
Tenability is our main target; the other two notions are introduced first because they isolate key components of non-uniform defense and help motivate the design choices in the final definition.

    We fix some terminology regarding argumentation frameworks. An argumentation framework $\F=(F,\att)$ will always be identified with its set of arguments $F$, i.e., we write $a\in \F$ to mean $a\in F$. 
    \begin{defn}\label{def: background definitions}
        Given an argumentation framework $\F$, argument $a\in \F$, and sets of arguments $A,B\subseteq \F$, we say:
        \begin{itemize}
            \item  $a$ \textit{attacks} a set of arguments $B$ ($a\att B$) if $(\exists b\in B)[a\att b]$;
            \item $A^+=\Set{x\in \F\ |\ (\exists a\in A)[ a \att x]}$;
            \item $A$ \textit{is as cogent as} $B$, written $A\succeq B$, if $A$ is conflict-free and every $b\in B$ which attacks $A$ is counterattacked by some argument in $A$: $(\forall b\in B)[b\att A\implies b\in A^+]$;
            \item $B$ \textit{is more cogent than} $A$ ($B\succ A$) if $B$ is as cogent as $A$, but $A$ is not as cogent as $B$.
        \end{itemize}
    \end{defn}

    \subsection{Static Tenability}\label{subsec: static tenability}
        One possible way of understanding the debate over $a\in \U$ is to recognize that for every consistent attack on $a$, there is a consistent extension containing $a$ which defends itself against said attack. In other words, a naïve defense strategy is one that has a reply for every set of incoming attacks. Thus, the minimal working definition to capture tenability is the following:

        \begin{defn}
            A set of arguments $A$ is \textit{statically tenable} if for every finite conflict-free $B$, there is a conflict-free $C\supseteq A$ which is as cogent as $B$.
        \end{defn}

        As a preliminary test, we analyze the examples $\S,\, \F,$ and $\U$ from Figure \ref{fig: main examples}. Static tenability accepts $a\in \S$ because there is no conflict-free attack to defend against, and rejects $a\in \F$ because there is a conflict-free $B=\Set{b_1}$ for which there is no defense. In the case of $\U$, there are $4$ conflict-free sets which attack $\Set{a}$. Namely, $\Set{b_1},\Set{b_1,c_2},\Set{b_2},$ and $\Set{b_2,c_1}$. The first two are handled by extending $\Set{a}$ to $\Set{a,c_1}$. The latter two are handled by extending $\Set{a}$ to $\Set{a,c_2}$. In any case, the proponent of $\Set{a}$ has a strategy to counter any consistent attack on her beliefs.

        This interpretation of tenability has some advantages and some drawbacks. The definition is concise and straightforward and represents a preliminary or heuristic check that an agent may apply to its beliefs.
        However, this definition does leave much to be desired. It merely represents the first exchange of a debate; it is a minimal requirement for an agent to defend its beliefs against incoming attacks. In a more realistic exchange, the attacker would continue its charge, bringing forth further arguments against the defending set. Consider for example the argumentation framework $\mathcal{E}$ (Figure \ref{fig: static tenable, not strongly tenable}). The set $\Set{a}$ can be checked to be statically tenable. However, in an extended dispute, $a$ has no adequate defense, hence should not be classified as tenable.

        The extended dispute may go as follows: the opponent to $a$ proposes $b_1$. In order to counterattack $b_1$, the proponent must extend its position by accepting $c_1$ or $c_2$. If she accepts $c_1$, then the opponent continues their attack with $b_3$. There is no way for the proponent to counter the attack $b_3$ without contradicting the positions she has already assumed during the debate. On the other hand, if she accepts $c_2$, then the opponent continues their attack with $b_2$. There is no consistent rebuttal to $\Set{b_1, b_2}$ which extends $\Set{a,c_2}$. The tables below are a representation of these debates as two player games between the proponent (\textit{Pro}) and her opponent (\textit{Opp}). 
        
    \begin{figure}[t]\centering
        \begin{tikzpicture}[thick, style= {scale=1.1}] 
            \node[arg] at (0, 0) (a){$a$};
            \node[arg] at (-1,1) (b1){$b_1$};
            \node[arg] at (0.5,1) (b2){$b_2$};
            \node[arg] at (1.5,1) (b3){$b_3$};
            \node[arg] at (-1.5,2) (c1){$c_1$};
            \node[arg] at (-0.5,2) (c2){$c_2$};
            \node[arg] at (0.5,2) (c3){$c_3$};
            \node[arg] at (1.5,2) (c4){$c_4$};
            
            
            \draw[->] (b1) to (a);
            \draw[->] (b2) to (a);
            \draw[->] (b3) to (a);
            \draw[<->] (b2) to (b3);
            \draw[->] (c1) to (b1);
            \draw[->] (c2) to (b1);
            \draw[->] (c3) to (b2);
            \draw[->] (c4) to (b3);
            \draw[<->] (c1) to [out= 45, in =135] (c4);
            \draw[<->] (c2) to [out= 45, in =135] (c3);
        \end{tikzpicture}
        \caption{An argumentation framework  $\mathcal{E}$ in which $\Set{a}$ is statically tenable, but not strongly tenable.}
        \label{fig: static tenable, not strongly tenable}
    \end{figure}
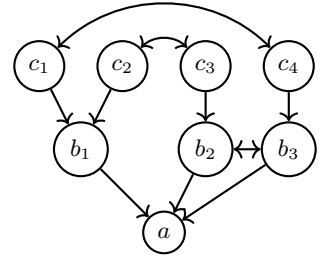

    \begin{figure}[h]
        \centering
        \begin{tabular}{c|cc}
            $\I$ & $\Set{a}$ & $\Set{a,c_1}$ \\\hline
            $\II$ & $\Set{b_1}$ & $\Set{b_1, b_3}$
        \end{tabular}\hspace{0.25cm}
        \begin{tabular}{c|cc}
            $\I$ & $\Set{a}$ & $\Set{a,c_2}$ \\\hline
            $\II$ & $\Set{b_1}$ & $\Set{b_1, b_2}$
        \end{tabular}        
            \label{fig: G game table}
    \end{figure}

    The important aspect we are highlighting is that the tenability of $\Set{a}$ is conditional upon having a \textit{winning strategy} in a certain dispute about $\Set{a}$. Arguments and debates involve back-and-forth, so tenable arguments should be maintainable beyond the first round of attacks. We thus isolate precisely which disputes are relevant and give \textit{dynamic} definitions for tenability.

\subsection{Strong Tenability}\label{subsec: strong tenability}

    \begin{defn}[Strong Tenability]
         Let $\F$ be an AF. A \textit{strong tenability dispute} on $\F$ is a sequence of subsets $(X_0,X_1,..., X_n)$  of $\F$, of which the even moves $(X_{2i})$ are $\I$  moves (for Proponent), and the odd moves $(X_{2i+1})$ are $\II$  moves (for Opponent), such that:
        \begin{enumerate}[(1)]
            \item each $X_i$ is conflict-free;
            \item $X_i\subseteq X_{i+2}$ for all $i$;
            \item $X_{1}$ and each $(X_{i+2}\setminus X_i)$ are finite;\footnote{Note that the definition allows for $X_0$ to be infinite, while all subsequent attacks and defenses may only introduce finitely many new arguments. Formalizing strong tenability disputes and tenability disputes in this manner facilitates the future study of tenability semantics in infinite Argumentation Frameworks. See section \ref{subsec: tenability infinite setting}.}
            \item every $\I$ move, $X_{2n}$, is as cogent as the preceding $\II$ move $X_{2n+1}$, i.e., $X_{2n}\succeq X_{2n-1}$;
            \item every $\II$ move, $X_{2n+1}$, has an attacker of the preceding  $\I$ move, $X_{2n}$, which is not counterattacked by $X_{2n}$, i.e., $X_{2n+1}\not\preceq X_{2n}$.
        \end{enumerate}
        The \textit{strong tenability game} is the alternating move game between $\I$ (or Proponent) and $\II$ (or Opponent) where the legal positions are the strong tenability disputes.
        
        A strong tenability dispute $(X_0,X_1,..., X_n)$ has \textit{concluded} if there is no $X_{n+1}$ such that $(X_0,X_1,..., X_{n+1})$ is a strong tenability dispute. A concluded strong tenability dispute is \textit{won} by the player making the last move $X_n$. A \textit{strategy for $\I$} is a function mapping each strong tenability dispute of even length to a legal move for $\I$, i.e. mapping $(X_0,\dots ,X_{2n+1})$ to some set $X_{2n+2}$ so that $(X_0,\dots X_{2n+1},X_{2n+2})$ is also a strong tenability dispute. We say that a strategy is a \textit{winning strategy for $\I$} if every concluded dispute where $\I$'s moves are determined by the strategy is won by $\I$. 
        A set $A\subseteq\F$ is \textit{ strongly tenable} if $\I$ has a winning strategy beginning with $X_0=A$. 
        Note that a winning strategy need not lead to a concluded dispute. Rather, if the strategy leads to an infinite sequence of moves, we consider this a victory for $\I$, since she has successfully defended $X_0=A$ against successive attacks.
    \end{defn}
    
    In the case that $\II$ plays last in a concluded dispute, this means that $\II$ has played a set with an attacker of $\I$'s last move which is not consistently counterattacked by any collection of arguments $I$ has played or can play. In the case that $\I$ plays last, $\I$ may only need to play a set of arguments `as good as' any further counterarguments. In a debate this may look like the opponent realizing he has no strictly better counterarguments, so he resigns his assault upon $A$. Thus strong tenability semantics, following the lead of Dung, are also guided by the principle ``the one who has the last word laughs best'' \cite[p. 322]{Dung1995}.
A further note is that backtracking is omitted as part of the dispute: $A$'s strong tenability being defined in terms of a winning strategy means that if backtracking is necessary for a player, then that player loses the run of the dispute and a winning strategy would avoid such plays. 

    Compare the strong tenability dispute to the argument game for admissibility (see \cite{VreeswijkPrakken2000}) where there is more asymmetry in the roles of proponent and opponent. The proponent must maintain consistency while the opponent need not (`Master versus Ignoramus'). It is exactly this asymmetry which is addressed through the family of tenability semantics. Indeed it underscores how admissibility requires a uniform response to challenges, whereas tenability does not. However, there is an asymmetry to our definition that persists:  the opponent does not need to defend the arguments they put forth. $\I$ must defend all of the arguments it incorporates into the dispute, while $\II$ may leave loose ends. For this reason, we find it plausible (indeed likely) that a rational agent defending its position $a$ will disregard the opponent's evidence if it is not also defended in the course of the dispute. Consider again the Strong Tenability dispute on $\mathcal{E}$, in the form of a dialogue:

    \begin{quote}
    \textit{$\I$: I propose $a$.\\
        $\II$: Well $b_1$ is true, so $a$ is not.\\ 
        $\I$: If $b_1$ is the position you hold, then I will defend $a$ by adopting $c_1$.\\
        $\II$: I now propose $b_3$, another reason $a$ cannot hold.\\
        $\I$: Now hold on, you haven't defended your own position $b_1$. I don't believe you are proposing your arguments in good faith. Thus I reject this threat upon $a$.}
    \end{quote}

    The lesson we take from this is that a rational agent may not  be forced to abandon her positions in the face of bad faith arguments. To be swayed from the original position, we should require the opponent to play by the same rules. If admissibility is captured by a game of `Master vs Ignoramus', then the strong tenability dispute is one of \textit{Master versus Novice} where the novice need not defend his positions, yet is expected to not take explicitly contradictory positions. We next define tenability in order to model a dispute of \textit{Master versus Master}.

\subsection{Tenability}\label{subsec: tenability}
    In addition to the desiderata suggested by $\F,\S,$ and $\U$, our investigation has yielded further refinements for an adequate notion of tenability:
    \begin{enumerate}[(i)]
        \item The proponent and opponent are allowed to progress the dispute by extending their current stance (accepting further arguments);
        \item Both proponent and opponent must argue \textit{in good faith}, that is, they must defend the positions they've proposed during the dispute;
        \item An argument is tenable when the proponent has a (possibly non-uniform) rebuttal to all lines of attack from the opponent, i.e. has a winning strategy.
    \end{enumerate}

    We now propose a refined definition of tenability incorporating these criteria. 
    \begin{defn}[Tenability]
        A \textit{tenability dispute} on $\F$ is a sequence of subsets $(X_0,X_1,..., X_n)$ of $\F$, of which the even moves $(X_{2i})$ are $\I$ moves, and the odd moves $(X_{2i+1})$ are $\II$ moves, such that:
        \begin{enumerate}[(1)]
            \item each $X_i$ is conflict-free;
            \item $X_i\subseteq X_{i+2}$ for each $i$;
            \item $X_{1}$ and each $(X_{i+2}\setminus X_i)$ are finite;
            \item every move must be as cogent as the preceding move ($X_{i+1}\succeq X_i$);
            \item every $\II$ move $X_{2i+1}$ must be more cogent than the preceding $\I$ move $X_{2i}$ ($X_{2i+1}\succ X_{2i}$).
        \end{enumerate}

        The \textit{tenability game} is the alternating move game between $\I$ (or Proponent) and $\II$ or (Opponent) where the legal positions are the tenability disputes. 
        A tenability dispute $(X_0,X_1,..., X_n)$ has \textit{concluded} if there is no $X_{n+1}$ such that $(X_0,X_1,..., X_{n+1})$ is a tenability dispute. A concluded tenability dispute is won by the player making the last move $X_n$. We say that a strategy is a \textit{winning strategy for $\I$} if every concluded dispute where $\I$'s moves are determined by the strategy is won by $\I$. A set $A\subseteq\F$ is \textit{tenable} if $\I$ has a winning strategy beginning with $X_0=A$.
    \end{defn}

    The distinction between the tenability dispute and the strong tenability dispute is found only in the symmetry between $\I$ and $\II$ in the fifth requirement (5). One asymmetry still remains, in that $\II$'s moves must contain an attacker of $\I$'s prior move which is not counterattacked, while $\I$ need only provide counterattacks. It is worth here justifying this asymmetry:

    Because we are interested in modeling the plausibility of an argument, rather than its provability, we are seeking an inherently defensive notion. Hence $\I$'s plays need only provide an adequate defense of $A$, rather than proactively proving $A$. This same motivation also indicates why $\II$ must provide a argument which attacks and is not attacked by $\I$'s move: if $\II$ does not provide a strictly stronger play $X_{2n+1}$, then $\I$'s current play $X_{2n}$ has already maintained itself against that charge. There is no good reason to adopt further beliefs to defend against $X_{2n+1}$ as $X_{2n}$ accounts for those doubts, hence $X_{2n+1}$ is not a move which progresses the dispute. Furthermore, we have no requirement on the number of arguments that $\I$ and $\II$ may include in a dispute, besides that they may introduce only finitely many new arguments, which means that both $\I$ and $\II$ have the appropriate means at every stage of the dispute to both defend their prior positions and, in the case of $\II$, put forward a new attack. The tenability dispute serves our intuitions about \textit{Master vs. Master} debates.

\section{Mathematical Analysis}\label{sec:math}

    The preceding discussion suggests that these notions of tenability are distinct, yet related. Indeed this is the case, as shown in the next proposition. Proofs are deferred to the supplementary material.

    \begin{restatable}{proposition}{TenabilityImplications}
\label{Prop:TenabilityImplications}
Strong tenability implies tenability; tenability implies static tenability. Moreover, each implication is strict.
\end{restatable}

    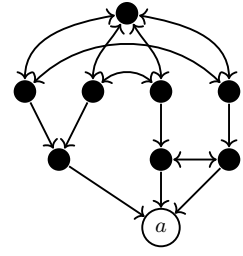
\begin{figure}
        \centering
        \scalebox{0.9}{
        \begin{tikzpicture}[thick,node distance =10mm, on grid ] 

                    \node[arg2] at (6.5, -10.5) (a){$a$};
            \node[dot2]  (b1) [above left =10mm and 15mm of a]{};
             \node[dot2] (b2)[above = 10mm  of a]{};
             \node[dot2] (b3)[above right =10mm and 10mm of a]{};
    
                \node[dot2] (c3) [above of = b2] {};
                \node[dot2] (c2)[left of = c3]{};
                \node[dot2] (c1)[left of =c2]{};
                \node[dot2] (c4) [above of = b3] {};

                \draw[->] (b1) to (a);
                \draw[->] (b2) to (a);
                \draw[->] (b3) to (a);
                \draw[<->] (b2) to (b3);
                \draw[->] (c1) to (b1);
                \draw[->] (c2) to (b1);
                \draw[->] (c3) to (b2);
                \draw[->] (c4) to (b3);
                \draw[<->] (c1) to [out= 45, in =135] (c4);
                \draw[<->] (c2) to [out= 45, in =135] (c3);
            
            \node[dot2] (d)[above left = 31.5mm and 5mm of a]{};
            \draw[<->] (d) to [out= 190, in=90](c1);
            \draw[<->] (d) to [out= 240, in=90](c2);
            \draw[<->] (d) to [out= 300, in=90](c3);
            \draw[<->] (d) to [out= 350, in=90](c4);

        \end{tikzpicture}}
        \caption{An argumentation framework in which $\Set{a}$ is statically tenable, but not tenable.}
        \label{fig:statten and not ten}
    \end{figure}

    The strictness of these implications are given by the framework in Figure \ref{fig:statten and not ten}, which is statically tenable, but not tenable, and the framework in Figure \ref{fig: static tenable, not strongly tenable}, which is tenable, but not strongly tenable.

    One concern for tenability and other weak semantics is their compatibility with existing semantics. The following proposition shows that the tenability notions are indeed weakenings of admissibility. 

    \begin{restatable}{proposition}{AdmissibilityImpliesTenability}
    \label{Prop:AdmissibilityImpliesTenability}
        If $A$ is admissible, then $A$ is (statically/-/strongly) tenable.
    \end{restatable}

    We also show that (strong) tenability is compatible with the grounded extension. Further, in finite AFs, static tenability is also compatible with the grounded extension.
    
    \begin{restatable}{theorem}{TenableUnionGrounded}
    \label{thm:TenableUnionGrounded}
         If $A$ is (strongly/-) tenable and $\G$ is the grounded extension, then $A\cup \G$ is (strongly/-) tenable. If $\F$ is finite and $A$ is statically tenable, then $A\cup \G$ is statically tenable. 
    \end{restatable}
         
    As an aside, there exists an infinite argumentation framework $\F$ and an argument $a$ so that $a$ is statically tenable, yet $a$ is attacked by the grounded extension. See Theorem \ref{thm: static tenable is not compact} and Figure \ref{fig: tenable not compact}.
    
    The next proposition highlights a structural property that sharply distinguishes tenability from most familiar weak semantics: \textit{downward closure} (or \textit{conservativity}) with respect to set inclusion.

    \begin{restatable}{proposition}{TenabilitySubsets}
        \label{prop:TenabilitySubsets}
        If $A$ is (strongly/-/statically) tenable, and $X\subseteq A$, then $X$ is (strongly/-/statically) tenable.
    \end{restatable}

At first sight this may look surprising, since many extension-based semantics are not closed under subsets: an extension is meant to be a self-standing position, and dropping arguments may destroy the internal support needed to defend the remaining ones.
Tenability is deliberately different.
Under our strategy-based reading, a set $A$ is not certified because it is a single stance that ``stands on its own'';
rather, $A$ is tenable because it can be \textit{maintained} against consistent attacks by suitably extending it, possibly in different ways for different attacks.
If $A$ can be maintained, then any weaker commitment $X\subseteq A$ can be maintained as well: starting from $X$ the proponent can simply follow the very same defense strategy, committing (when needed) to some of the additional arguments that were available from $A$.
Intuitively, moving from $A$ to $X$ removes liabilities rather than resources: any consistent attacks on $X$ were already attacks on $A$, and it cannot make defense harder, since the proponent is always allowed to add further arguments later.
In this sense, conservativity is not an artefact but an intrinsic consequence of modeling acceptance as maintainability under monotone commitments. 

A useful corollary is that credulous tenability of an argument $a$ is equivalent to the tenability of the singleton $\{a\}$. 
Accordingly, tenability is best understood as a semantics of \textit{maintainable commitments} rather than of \textit{complete positions}.

    One may wonder about the effect of restricting the number of arguments each player can introduce at each stage of the debate. In the case of strong tenability, this has absolutely no impact on which player has a winning strategy. Strong tenability is robust with respect to play size restrictions. This indicates that the winning motto of strong tenability is to play as little as possible, so as not to be forced into contradiction. 

    \begin{restatable}{theorem}{StrongRobustness}\label{thm: strong robustness}
        $A\subseteq\F$ is strongly tenable if and only if $\I$ has a winning strategy starting with $A$ in the strong tenability game where each player may only introduce $1$ new argument in each play.
    \end{restatable}

    Tenability is not robust in this sense. Consider $\Set{a}$ of the argumentation framework given in Figure \ref{fig: 1tenable not tenable}. The winning strategy to refute $a$ in the tenability game, playing $\Set{b_1,b_2}$, is not available when we restrict each player to advancing a single argument at a time. If $\II$ attacks $a$ with $b_1$, then he gets stuck defending $b_1$ from $c_1$. He would need to play $c_2$, but this is not a play \textit{more cogent than}  $\{a,c_1\}$, thus is not a legal move. Even if we were to allow $\II$ to play $c_2$, this would nullify the threat of $b_2$, highlighting that in the tenability game, $b_1$ and $b_2$ need to be played simultaneously for $\II$ to win.

    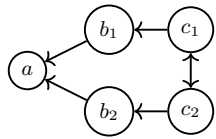
\begin{figure}[h]
        \centering
        \scalebox{0.9}{
        \begin{tikzpicture}[thick,node distance =12mm, on grid ] 
            \node[arg2] (a) at (5,-6) {$a$};
            \node[arg2] (b) [above right = 6mm and  12mm of  a] {$b_1$}; 
            \node[arg2] (c) [below of = b] {$b_2$};
            \node[arg2] (d) [right of= b] {$c_1$};
            \node[arg2] (e) [below of= d] {$c_2$};
            \draw[->] (b) to (a);
            \draw[->] (c) to (a);
            \draw[->] (d) to (b);
            \draw[->] (e) to (c);
            \draw[<->] (d) to (e);

        \end{tikzpicture}}
        \caption{An AF demonstrating the need to play sets of arguments simultaneously.}
        \label{fig: 1tenable not tenable}
    \end{figure}

    We end this portion of the investigation by classifying the complexity of each semantics.
    
\newpage
    \begin{restatable}{theorem}{Complexity}
    \label{thm:Complexity}
        For finite $\F$:
        \begin{enumerate}
            \item The problem of determining if $A\subseteq \F$ is statically tenable is  $\Pi^P_2$-complete.
            \item The problem of determining if $A\subseteq \F$ is tenable is $PSPACE$-complete.
            \item The problem of determining if $A\subseteq \F$ is strongly tenable is  $PSPACE$-complete.
        \end{enumerate}
    \end{restatable}

    In terms of complexity, this shows that tenability and strong tenability are on par with weak admissibility which is also PSPACE complete \cite{DvorakUlbrichtWoltran2022}. Static tenability is on par with skeptical acceptance for the preferred semantics which is also $\Pi^P_2$ \cite{DunneBenchCapon-SkeptPr}. This is evidence for the relative practical applicability of these new semantics that their complexity for implementation is on par with existing semantics.

\subsection{Tenability in the infinite setting}\label{subsec: tenability infinite setting}
Though our main focus is on finite argumentation frameworks, we have made definitions for (strong/-/static) tenability which also work well in the infinite setting, and the previous results hold for infinite argumentation frameworks as well.
We note that in each definition, we make the choice to allow the set of arguments $A$ under discussion to be infinite, yet we demand that at each stage, each player may only play finitely many new elements. We note that alternate definitions could be made, but we choose the most concrete definition, hewing closely to the goal of modeling debate via dialog. In a debate, for $A$ to be conclusively refuted, it must be refuted via a finite-length debate using finitely many arguments at each stage. 

In set theory, games are considered between two players where alternating moves are used to construct some sequence $(X_0,X_1,\dots)$ and the victory condition for each player is determined by membership of the sequence $(X_0,X_1,\dots )$ in some set $\mathcal S$. A particular pitfall in this area is that, as a result of the axiom of choice, there can exist games so that neither player possesses a winning strategy \cite[Chapter 33]{Jech2003}, or even games that have winning strategies for different players, depending on the chosen model of set theory \cite{LarsonShelah}. 
This is distinctly not the case for the games we use to define (strong/-/static) tenability. Our notions are well-defined and no set theoretic problems arise:

\begin{restatable}{theorem}{Determinacy}
    If $A\subseteq \F$ is any argumentation framework, then either $\I$ has a strategy showing that $A$ is (strongly/-/statically) tenable or $\II$ has a strategy showing that $A$ is not (strongly/-/statically) tenable.
    Further, the (strong/-/static) tenability of $A\subseteq \F$ does not depend on the model of set theory used to evaluate its (strong/-/static) tenability.
\end{restatable}

We also note that it is necessary to consider the tenability of infinite sets $A\subseteq \F$ in order to model consistently simultaneously defending the set $A$ of beliefs. In particular, this is strictly stronger than consistently defending each finite subset. The argumentation framework given in Figure \ref{fig: tenable not compact} demonstrates this. The set in question is $a$, together with the unattacked arguments.

\begin{theorem}\label{thm: static tenable is not compact}
    There exists $A\subseteq \F$ so that $A$ is not statically tenable, yet every finite subset of $A$ is strongly tenable. 
\end{theorem}

    \begin{figure}[h] 
        \centering
        \scalebox{0.9}{
        \begin{tikzpicture}[thick,node distance =7mm, on grid ] 
            \node[arg2] (a) at (0,0) {$a$};
            \node[arg2] (w) [left =15mm of a] {$\omega$};
            
            \node[dot2] (b0) [above left = 10mm and  15mm of  w] {}; 
            \node[dot2] (b1) [above of = b0] {};
            \draw[->] (b0) to[out=-50,in=160] (w);
            \draw[->] (b1) to (b0);
            \node (dummy) [left of = b0] {};
            
            \node[dot2] (b0) [above left = 10mm and  7.5mm of  w] {}; 
            \node[dot2] (b1) [above of = b0] {};
            \draw[->] (b0) to (w);
            \draw[->] (b1) to (b0);
            
            \node[dot2] (b0) [above= 10mm and  of  w] {}; 
            \node[dot2] (b1) [above of = b0] {};
            \draw[->] (b0) to (w);
            \draw[->] (b1) to (b0);
            \node (dots) [right =7.5mm of b1] {$\cdots$};
            \node (dots) [right =7.5mm of b0] {$\cdots$};
            
            \node[dot2] (b0) [above right = 10mm and  15mm of  w] {}; 
            \node[dot2] (b1) [above of = b0] {};
            \draw[->] (b0) to (w);
            \draw[->] (b1) to (b0);
            \node (dots) [right =7.5mm of b1] {$\cdots$};
            \node (dots) [right =7.5mm of b0] {$\cdots$};

            \draw[->] (w) to (a);

        \end{tikzpicture}}
        \caption{An argumentation framework with a set $A$, consisting of $a$ along with all unattacked arguments, which is not statically tenable, yet every finite subset is statically tenable.}
        \label{fig: tenable not compact}
    \end{figure}
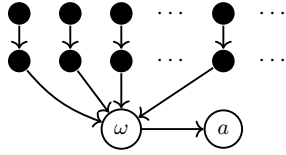

\section{Comparison to other semantics}\label{sec:other semantics}

    Tenability departs radically from  labeling-based,
    extension-based, and recursion-based semantics. There is no way to label arguments \tin, \tout, \tun that coincides with our intuitions about tenability. The core idea of a non-uniform defense strategy is not naturally captured by a single labeling. Further, the PSPACE-completeness indicates that tenability cannot be captured by the existence of a certain labeling with some simple properties. The dynamic mechanism of tenability thus separates it from many other weak semantics, but it is included with game-based weak semantics. Here, we will situate tenability among a few semantics that have played key roles in the recent literature. 
    
    The \textit{cogency} semantics of Bodanza, Tohmé, and Simari \cite{BodanzaTohmeSimari2016} were introduced to directly address issues of self-defeating attack and acceptance in odd-length cycles. They do so by implementing the \textit{as cogent as} relation which allows for comparison between sets of arguments: one set $X\subseteq \F$ is as cogent as another $Y\subseteq\F$, if $X$ is admissible in the restriction of $\F$ to $X\cup Y$ (see Definition \ref{def: background definitions}). The isolated comparison of sets of arguments via the cogency relation naturally corresponds to evaluating those sets on how they may be argued against each other in a dialogue. A set is \textit{cogent} if it is undominated with respect to this relation. A weaker notion, \textit{cyclic cogency} relaxes this further in anticipation of the issues presented by odd-length cycles.

    \begin{defn}
        Let $\F$ be an argumentation framework and  $X\subseteq \F$. $X$ is \textit{cogent} if for every $Y\subseteq \F$, $Y\succeq X$ implies $X\succeq Y$. $X$ is \textit{cyclically cogent} if for every $D\subseteq\F$ at least one of the following holds: (1) $D\not\succeq X$; (2) $X\succeq D$;
            (3) there exists a chain $D=D_1\prec D_2\prec \dots \prec D_n=X$.
    \end{defn}
    The cogency relation plays a key role in defining tenability, however, (strong/-/static) tenability semantics require strictly less than cogent semantics; where the latter requires a fixed singular line of defense for a set of arguments, tenability semantics relaxes this to allow for the defense to change, dependent upon the attacks presented. We note that cogency and cyclic cogency are also captured by game-based semantics \cite{BodanzaTohmeSimari2016}.

    A different approach to defending against illegitimate arguments lies in the recursion-based weak admissibility semantics of Baumann, Brewka, and Ulbricht \cite{BaumannBrewkaUlbricht2020}. Informally, a set of arguments $X$ is ruled acceptable whenever its attackers are rendered unacceptable after removing the arguments attacked by $X$. 
    \begin{defn}
            For any argumentation framework $\F$ and $A\subseteq \F$, $\F^A$ is the restriction of $\F$ to $F\smallsetminus (A\cup A^+)$, called the $A$-\textit{reduct} of $\F$.  A set $A\subseteq \F$ is \textit{weakly admissible} if it is conflict-free and for every attacker $y$ of $A$, $y$ belongs to no weakly admissible set of $\F^A$.
        \end{defn}
        The last weak semantics we will consider is weakly complete semantics introduced by Dondio and Longo \cite{DondioLongo2021}. Weakly complete semantics employs a different mechanism for evaluating acceptance, and capture acceptance in scenarios where undecided arguments  may play lessened roles.
        \begin{defn}
        An extension is \textit{weakly complete} if it is the set of \texttt{in}-labeled arguments in a \textit{weakly complete labeling} of $\F$. Such a labeling is a map from $\F$ to $\{\tin,\tout,\tun\}$ where for every $a\in \F$, each of the following holds:
        (1) if $a$ is labeled \texttt{in}, then there are no attackers of $a$ labeled \texttt{in};
        (2) if $a$ is labeled \texttt{out}, then it has at least one attacker labeled \texttt{in};
        (3) if $a$ is labeled \texttt{undec}, then it has at least one attacker labeled \texttt{undec} and it has no attackers labeled \texttt{in}.

    \end{defn}

    Weakly complete semantics are by far the weakest semantics in this realm with respect to which arguments are credulously accepted. However, it is not the case that every, e.g., cogent extension is weakly complete, rather, every cogent extension is a subset of a weakly complete extension. We next introduce a formal method of capturing this relation between semantics and show all possible relationships between the semantics discussed in this paper.
    
    \begin{defn}
        Let $\sigma$ and $\tau$ be semantics. We say that $\sigma$ is an inclusion-based refinement of $\tau$ if every $\sigma$-extension is a subset of a $\tau$-extension.
    \end{defn}

    Observe that if $\sigma\subseteq \tau$, then $\sigma$ is also an inclusion-based refinement of $\tau$.

    \begin{restatable}{theorem}{WeakSemanticsZoo}
    \label{thm: weak semantics zoo}
        For finite argumentation frameworks, Figure \ref{fig: weak semantics summary-diagram} represents the implications and inclusions between weak semantics. 
        Each solid arrow indicates implication. Each dashed arrow indicates inclusion-based refinement. If there is no directed path from one semantics to another, then neither implication holds.
    \end{restatable}

\newcommand{\tikzfinitezoo}{

    \scalebox{0.9}{
        \begin{tikzpicture}[
          node distance=6mm,
          box/.style={draw,rounded corners,inner sep=3pt,align=center},
          arr/.style={-Stealth,thick}]
        ]

        \node[box] (cog) {Cogent};
        \node[box, right=of cog] (strten) {Strongly\\ Tenable};
        \node[box, below right = 7.8mm and 12mm of cog] (wad) {Weakly\\Admissible};
        \node[box, below =7.5mm of cog] (cyc) {Cyclically\\Cogent};

        \draw[arr] (cog)--(cyc);
        \draw[arr, out=340,in=160] (cog) to (wad);
        
        \node[box,right=of strten] (ten) {Tenable};
        \node[box,right=of ten] (staten) {Statically\\Tenable};
        \node[box, below  =  of staten] (wcomp) {Weakly\\Complete};
        
        \draw[arr] (cog)--(strten);
        \draw[arr] (strten)--(ten);
        \draw[arr] (ten)--(staten);

        \draw[dashed,-Stealth, thick] (staten)--(wcomp);
        \draw[dashed,-Stealth, thick] (wad) to  (wcomp);

        \end{tikzpicture}}

}
    \begin{figure}[h!]
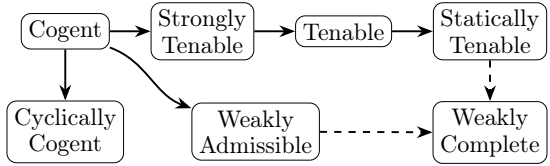

        \centering
        \tikzfinitezoo
        \caption{The relations among weak semantics in finite argumentation frameworks.
        }
        \label{fig: weak semantics summary-diagram}
    \end{figure}

    We remark that some of the refinements among semantics presented in Theorem \ref{thm: weak semantics zoo} need not persist in infinite argumentation frameworks. In fact, semantics such as weak admissibility fail to be well-defined for arbitrary argumentation frameworks.\footnote{Consider whether $a_0$ is weakly admissible in $\AF$  where $\Args =\{a_i:i\in\mathbb{N}\}$ and $a_j\att a_i$ for all $i<j$.} We leave the project of mapping out the relations among weak semantics in the arbitrary setting to future work.

\section{Discussion}\label{sec:discussion}\label{sec: discussion}

This paper introduced three tenability notions---\textit{static tenability}, \textit{strong tenability}, and \textit{tenability}---all centered on the same conceptual shift: defensibility is not necessarily witnessed by a single uniform stance, but may consist in the ability to respond \textit{non-uniformly} to any consistent line of attack.
In this sense, tenability adds a new axis to the weak-semantics landscape: it preserves a defense-based intuition while making strategic non-uniformity explicit at the abstract level.

Because our definitions are purely abstract, tenability is not intended as a standalone model of practical reasoning in rich domains.
In applications, the main modeling work lies upstream: extracting an AF from structured arguments, rules, evidence, and preferences.
Well-established structured argumentation formalisms (and their instantiations into abstract AFs) therefore provide a natural interface,
including systems developed with legal and evidential reasoning in mind \cite{Prakken1997LegalArgument,ModgilPrakken2014ASPIC,GordonPrakkenWalton2007}.
Within such pipelines, tenability can be viewed as a \textit{dialectical evaluation layer} that targets a specific phenomenon:
whether a position can be sustained against consistent lines of attack without requiring a uniform defense against mutually incompatible challenges.

Adversarial legal procedure is an obvious arena in which the distinction between uniform and non-uniform defense matters,
since legal debates intertwine \textit{presumptions}, \textit{burdens of proof}, and varying \textit{standards of proof}
\cite{GordonPrakkenWalton2007,PrakkenSartor2006,Walton2014Burden,Prakken2009Burdens}.
Tenability does not aim to encode such standards directly; rather, given an AF extracted from an evidential record (possibly via a structured formalism),
it can support analyses of whether a party's position is dialectically maintainable against legitimate challenges, allowing conditional rebuttals of the kind highlighted by $\U$.
In particular, \textit{static tenability} naturally matches ``one-shot'' settings where an opponent advances an internally consistent body of evidence and the proponent must be able to extend her commitments with a consistent counter-case.
This resonates with the organisation of evidential reasoning around competing hypotheses or ``case theories'', and with hybrid approaches combining argumentative and narrative structure \cite{BexKoppenPrakkenVerheij2010}.

The game-based notions---\textit{strong tenability} and (especially) \textit{tenability}---are better suited to modeling extended exchanges in which commitments accumulate over time.
Here, monotone commitment protocols capture the idea that argumentative moves incur liabilities that cannot simply be ignored later,
a theme studied extensively in dialogue theory and commitment-based models of argumentation \cite{WaltonKrabbe1995,Prakken2005}.
In this setting, the distinction between strong tenability and tenability is not merely technical:
it separates debates in which one party is held to a stricter standard of sustained defensibility from debates in which both parties are held to this higher standard.
This provides a principled handle on the \textit{Master versus Ignoramus} effect isolated in Section~\ref{sect:motivation}:
strong tenability blocks spurious attacks, giving a suitable model of \textit{Master versus Novice} settings, whereas 
tenability blocks forms of ``attack switching'' that would be unreasonable under a defensibility requirement on lines of challenge,
while still allowing the proponent's defense to be non-uniform, thus giving a suitable model of \textit{Master versus Master} debates.

Outside legal domains, tenability seems relevant whenever deliberation proceeds against a background of defaults and asymmetric costs of revision.
In many practical settings, an agent has a current stance that is not abandoned unless a challenger can articulate a plausible defeating case,
an idea closely connected to presumption and burden allocation \cite{Walton2014Burden,Prakken2009Burdens}.
This suggests potential uses in interactive explanation and decision support:
a system may justify maintaining a position by exhibiting a defense strategy tailored to a user's reasonable objections,
without committing to a single all-purpose extension that anticipates mutually incompatible challenges.

\subsection*{Conclusions and future work}
Our comparison with existing weak semantics (Theorem~\ref{thm: weak semantics zoo} and Figure~\ref{fig: weak semantics summary-diagram}) does not exhaust the landscape.
A natural next step is to extend the map to additional weak semantics---for instance, the weakly-preferred and UB-preferred semantics studied by Dondio and Longo \cite{DondioLongo2021}.
Another direction concerns principle-based evaluation.
Since each tenability notion is closed under subset, several standard principles for extension-based semantics will fail for reasons that are orthogonal to the intended reading.
Two ways to make such a study informative are to analyze \textit{maximal} variants of the tenability notions, and to extract from tenability new principles that isolate non-uniformity as a design dimension applicable beyond our specific games. 

A particularly promising direction concerns the extension of tenability to richer argumentation formalisms, where the strategic character of defense becomes even more salient. Bipolar Argumentation Frameworks (BAFs) \cite{amgoud2008bipolarity} are a natural target: by distinguishing attack from support, BAFs better capture the structure of real dialogues, in which a proponent typically marshals not only counterattacks but also corroborating arguments. Lifting tenability to this setting, however, is not a routine exercise. The design choices that make tenability work in the purely conflict-based setting---legitimacy as conflict-freeness, monotone commitment, unanswered attack as the trigger for proponent moves, counterattack as the source of vindication---must all be rethought once support is present. One must decide, for instance, whether  commitments should be closed under supported acceptance, whether the proponent may introduce supporters of arguments already in play as a distinct kind of move, and whether the opponent may legitimately attack a position by attacking its supporters. Analogous questions arise for structured formalisms such as ASPIC+ \cite{ModgilPrakken2014ASPIC}, where tenability's strategy-based, monotone character appears especially well-suited to incremental dialectical evaluation.

Looking forward, a promising application of this work lies in the generation of transparent explanations in legal AI. Implementing tenability semantics in automated legal reasoning tools offers a natural solution for this domain, where courtroom disputes inherently require non-uniform, strategic defenses tailored to specific lines of attack. By translating the monotone commitment games introduced here into applied computational frameworks, we aim to develop dialectical tools that assist legal practitioners in evaluating the defensibility of complex claims and formulating targeted, context-aware rebuttals.

To conclude, tenability isolates and formalizes \textit{non-uniform} defense against consistent attacks as a principled alternative to uniform-witness notions of acceptance. We expect this perspective to be most useful when combined with established modeling layers for structure, priorities, and proof standards, rather than as a replacement for them.

\section{Acknowledgements}
The first author's work was supported by the National Science Foundation under Grant DMS-2348792. San Mauro is a member of INDAM-GNSAGA. 

\section{AI Declaration}
AI-based language tools were used for minor editorial assistance. No research content or results were generated using these tools.

\bibliographystyle{kr}
\bibliography{kr-sample}

\clearpage
\section{Supplementary Material}
\subsection{Omitted Proofs in Section \ref{sec:math}}
\def\sim{^{s}}

\TenabilityImplications*
\begin{proof}
    Note that any strategy for $\I$ which wins the strong tenability game is also a strategy for the tenability game, since the conditions on $\I$'s moves are the same, yet $\II$ has a restricted set of playable moves in the tenability game. Let $\Gamma$ be a winning strategy for the tenability game. Let $B$ be any finite conflict-free set. Then $\Gamma$ gives $\I$ a move to play against the move $X_1=B$. This is a set $C$ which contains $A$ and is as cogent as $B$. Thus $A$ is statically tenable.

    Note that $a\in \F_{7}$ is tenable but is not strongly tenable, while $a\in \F_{8}$ is statically tenable but is not tenable. See Table \ref{tab: cred acceptance} for these examples. This establishes the strictness.
\end{proof}

\AdmissibilityImpliesTenability*
\begin{proof}
    Consider the strong tenability game for $A$. $\II$ has no move on his first turn, since $A$ counterattacks each of its attackers. By \ref{Prop:TenabilityImplications}, this establishes that $A$ is tenable and statically tenable as well.
\end{proof}

\TenabilitySubsets*
\begin{proof}
    Suppose that $A$ is statically tenable and $X\subseteq A$. Let $B$ be any finite conflict-free set of arguments. Then there exists a conflict-free $C\supseteq A$ which is as cogent as $B$. Since $C\supseteq A\supseteq X$, $C$ witnesses the static tenability of $X$ as well.

    Suppose that $A$ is (strongly/-) tenable. Let $\Gamma$ be a strategy for $\I$ in the game for $A$. We now give a strategy for $\I$ in the game for $X$.
    We consider a position $(X_0,X_1,\dots X_{2n})$ to be \textit{safe} if $(A\cup X_0,X_1, A\cup X_2, X_3 ,\dots , A\cup X_{2n})$ is a play of the game for $A$ following the strategy $\Gamma$. 
    We need only establish that from a safe position $(X_0,X_1,\dots X_{2n})$, for any play $X_{2n+1}$, we can find another safe position $(X_0,X_1,\dots X_{2n},X_{2n+1},X_{2n+2})$ for $\I$ to move to. This gives our strategy for $\I$ in the game for $X$, since the initial position $(X)$ is safe. 

    Let $(X_0,X_1,\dots X_{2n})$ be a given safe position. Let $X_{2n+1}\not\preceq X_{2n}$ be conflict-free and contain $X_{2n-1}$. Then we consider $\Gamma$'s response $Y=Y'\cup A$ from the position $(A\cup X_0,X_1,A\cup X_2, X_3 ,\dots , A\cup X_{2n}, X_{2n+1})$. Then $Y\supseteq A\cup X_{2n}$ and $Y$ is as cogent as $X_{2n+1}$. Let $A_0$ be a finite subset of $A$ which attacks each element of $X_{2n+1}$ which $A$ attacks. This exists since $X_{2n+1}$ is finite. Then we let $X_{2n+2}=X_n\cup Y'\cup A_0$. Observe that this is conflict-free since it is a subset of $Y$, it is as cogent as $X_{2n+1}$, and it is safe. 
\end{proof}

\subsubsection{Tenability and the grounded extension}

We consider the following layered presentation of the grounded extension (See \cite{ASMTark} for a detailed analysis of this layered presentation of the grounded extension, including how many layers are needed in the infinite setting):

\begin{defn}
    Let $\G_0$ be the empty set. For each $n$, let $\G_{n+1}$ be the set of arguments which are defended by $\G_n$. That is, $a\in \G_{n+1}$ if $\forall c\in \F (c\att a \Rightarrow \exists b\in \G_n (b\att c))$.

    For $\F$ finite, the grounded extension $\G$ is $\bigcup_{n\in \mathbb N} \G_n$.
\end{defn}

For infinite argumentation frameworks, we need to generalize this definition.

\begin{defn}
    Let $\G_0$ be the empty set. 
    
    For each ordinal $\alpha$, let $\G_{\alpha+1}$ be the set of arguments which are defended by $\G_\alpha$. That is, $a\in \G_{\alpha+1}$ if $\forall c\in \F (c\att a \Rightarrow \exists b\in \G_\alpha (b\att c))$.

    For each limit ordinal $\lambda$, let $\G_\lambda=\bigcup_{\alpha<\lambda}\G_\alpha$.

    Then there is some $\beta$ so that $\G_\gamma=\G_\beta$ for all $\gamma \geq \beta$. The grounded extension is $\G_\beta$.
\end{defn}

A standard property of the ordinals which we use below is that every non-empty set of ordinals has a least member (see \cite[Chapter 2]{Jech2003}). In particular, if $(\alpha_n)$ is a descending sequence of ordinals, then it must have a least member, showing that the sequence has finite length.

\TenableUnionGrounded*
    \begin{proof}
        We first begin with a Lemma showing that tenability of $A$ implies that $A\cup \G$ is conflict-free. 

        \begin{lemma}\label{lem:tenable implies cf with grounded}
            If $A$ is tenable, then $A\cup \G$ is conflict-free.
        \end{lemma}
        \begin{proof}
            Suppose otherwise that there is a conflict between $A$ and $\G$. Then we know that some argument $g_1\in \G$ attacks some argument $a_1\in A$, since $\G$ is admissible. Now we give a strategy for $\II$ which proves the non-tenability of $A$. $\II$ on his first turn, plays $g_1$. In response, $\I$ must play some set including an element $a_2$ which attacks $g_1$. $\II$ responds playing an element $g_2$ which is in the grounded extension, attacks $a_2$, and enters the grounded extension before $g_1$ does. This means that if $g_1\in \G_\alpha$ for some $\alpha$, then $g_2\in \G_\beta$ for some $\beta<\alpha$. 
            
            Generally, at the $n$th turn of the game, $\I$ must play a set which includes an element $a_{n}$ which attacks $g_{n-1}$, and player $\II$ responds by playing a single element $g_n$ which attacks $a_n$, is in the grounded extension, and enters the grounded extension before $g_{n-1}$ does.

            We must argue that this game ends in finite time, with $\I$ failing to find a move. Since every element of the grounded extension is defended by the set of elements which enter the grounded extension before it, $\II$ can always find a move, so it suffices to show that the game ends in finite time. If we let $\alpha_n$ be the least ordinal so that $g_n\in \G_{\alpha_n}$, then $(\alpha_n)$ is a descending sequence of ordinals, thus must be finite.
        \end{proof}

        In what follows, we call a position $(X_0,X_1,\dots , X_n)$ winning if there exists a winning strategy for $\I$ beginning in this position, i.e., every concluded dispute which starts with this position and is consistent with the strategy is won by $\I$.

        We now give a strategy for $\I$ to play the (strong/-) tenability game for the set $A\cup \G$, and we demonstrate that this strategy must win, thus establishing that $A\cup \G$ is (strongly/-) tenable. 
        We maintain, by induction, that at each position $(X_0,X_1,\dots , X_{2n})$ that $\I$ plays in the game, the corresponding position $(\hat{X}_0, X_1',\hat{X}_2, X_3', \hat{X}_4, \dots , X_{2n-1}', \hat{X}_{2n})$ is a winning position in the game for $A$, where $X_{k}'=X_k\smallsetminus \G^+$ and $\hat{X}_k=(X_k\smallsetminus \G)\cup A$. This is clearly true at $n=0$, as the position $(A)$ is a winning position by hypothesis.

        Suppose this is true for $n-1$ and we show that we can maintain this condition at $n$. At the position $(X_0,X_1,\dots , X_{2n-1})$ in the game for $A\cup \G$, we consider the corresponding position 
        $(A, X_1',\hat{X}_2, X_3', \hat{X}_4, \dots , X_{2n-1}')$ in the game for $A$. We have, by inductive hypothesis, that there exists some set $Y$ so that $(A, X_1',\hat{X}_2, X_3', \hat{X}_4, \dots , X_{2n-1}', Y)$ is a winning position in the game for $A$. In the game for $A\cup \G$, we let $X_{2n}=Y\cup \G$.

        We note that $X_{2n}$ defends against each attack in $X_{2n-1}'$, and each element of $X_{2n-1}$ is either in $X_{2n-1}'$ or in $G^+$, thus $X_{2n}$ has responded to each attack in $X_{2n-1}$. This strategy can only fail if, at some stage, $X_{2n}$ has a conflict. Suppose towards a contradiction that this were the case. Then we consider the position $(A, X_1',\hat{X}_2, X_3', \hat{X}_4, \dots , X_{2n-1}',\hat{X}_{2n})$ in the (strong) tenability game for $A$. At this point, some element $g\in \G$ attacks an element of $\hat{X}_{2n}$. No element of any $X_i'$ is in conflict with any element of $\G$, thus we can use the strategy from Lemma \ref{lem:tenable implies cf with grounded} to show that there is a winning strategy for $\II$ starting in this position. In particular, from here on, $\II$ plays only single element extensions, where the single new element is in $\G$ and can guarantee a win for $\II$, contradicting $(A, X_1',\hat{X}_2, X_3', \hat{X}_4, \dots , X_{2n-1}',\hat{X}_{2n})$ being a winning position for $\I$.

        Now we consider the case where $\F$ is finite. Then we have the following version of Lemma \ref{lem:tenable implies cf with grounded}:

        \begin{lemma}\label{lem:cant defend against grounded}
            If $C$ is as cogent as $B\supseteq \G$, then no element of $C$ is attacked by any element of $\G$. 
        \end{lemma}
        \begin{proof}
            Suppose towards a contradiction that some element of $C$ is attacked by an element of $\G$. Let $n$ be least so that some element $g_n\in\G_n$ attacks an element of $C$. Then some element $c\in C$ must attack $g_n$ in defense. But then some element of $\G_{n-1}$ attacks $c$, contradicting the minimality of $n$.
        \end{proof}

        \begin{lemma}\label{lem:statically tenable implies conflict-free with grounded}
            If $\F$ is finite and $A$ is statically tenable, then $A\cup \G$ is conflict-free.
        \end{lemma}
        \begin{proof}
            Suppose towards a contradiction that $A$ is statically tenable and in conflict with $\G$. In particular, some element of $\G$ attacks some element of $A$. By static tenability of $A$, there must be some set $C\supseteq A$ which is as cogent as $B=\G$, which is impossible by Lemma \ref{lem:cant defend against grounded}. Note that this argument uses the assumption that $\G$ is finite.
        \end{proof}

        Now we argue that $A\cup \G$ must be statically tenable, given that $\F$ is finite and $A$ is statically tenable. Let $B$ be any conflict-free set. We must find a set $C\supseteq A\cup \G$ as cogent as $B$. Let $B'=(B\smallsetminus \G^+)\cup \G$. There must be some set $C'\supseteq A$ which is as cogent as $B'$. By Lemma \ref{lem:cant defend against grounded}, no element of $C'$ is in $\G^+$.

        Let $C=\G\cup C'$. We argue that $C$ is conflict-free and is as cogent as $B$, proving that $A\cup \G$ is statically tenable. Since $C'\cap \G^+=\emptyset$, $C$ is conflict-free. We now consider attacks from $B$. If $b\in B$ is in $B\smallsetminus \G^+$, then $b$ does not attack $\G$, so $b$ can only attack $C'$ and $C'$ counterattacks against all attacks from $B'\supseteq B\smallsetminus \G^+$. If $b\in B\cap \G^+$, then $\G$ attacks $b$. Thus $C$ is as cogent as $B$, and we have established that $A\cup \G$ is statically tenable.

        Finally, we provide an example of an infinite argumentation framework with a statically tenable element which is in conflict with the grounded extension. See figure \ref{fig:staticly tenable but conflict grounded}. The element $a$ is easily seen to be statically tenable. For any finite set $B$ containing $\omega$ and finitely many more elements, let $n$ be large enough that $a_{2n,1}$ is not attacked by $B$. Then $C=\{a,a_{2n,1}\}$ is as cogent as $B$, establishing that $a$ is statically tenable. Yet $\omega$ is in $\G$ (in fact in $\G_{\omega+1}$) and attacks $a$.

\begin{figure}[h]
    \centering
    \begin{tikzpicture}
        [thick,node distance =15mm, on grid, 
        arg/.style={draw, circle,font=\footnotesize, minimum size = 10mm },
        ] 
            \node[arg2] (a) at (0,0) {$a$};
            \node[arg2]  (w) [above of = a] {$\omega$};

            \node[ov] (a41) [above left = 25mm and 7.5mm = w]{$a_{4,1}$};
            \node[ov] (a42) [above of = a41] {$a_{4,2}$};
            \node[ov] (a43) [above of = a42] {$a_{4,3}$};
            \node[ov] (a44) [above of = a43] {$a_{4,4}$};
            \node[ov] (a21) [left of  = a41] {$a_{2,1}$};
            \node[ov] (a22) [above of = a21] {$a_{2,2}$};

            \node[lab] (cdots) [right of = a41] {\large $\cdots$};
            \node[lab] [above =30mm of cdots] {\large $\cdots$};
            \node[ov] (an1) [ right of= cdots] {$a_{2n,1}$};
            \node[lab] (dots) [above  =25mm of an1] {\large$\vdots$};
            \node[ov] (ann) [above =25mm of dots]{$a_{2n,2n}$};
    
            \draw[->] (a22) to (a21);
            \draw[->] (a21) to (w);
            \draw[->] (a44) to (a43);
            \draw[->] (a43) to (a42);
            \draw[->] (a42) to (a41);
            \draw[->] (a41) to (w);
            \draw[->] (w) to (a);

            \draw[->] (ann) to (dots);
            \draw[->] (dots) to (an1);
            \draw[->] (an1) to (w);

             \node[lab] [right =12mm of an1] {\large $\cdots$};
    \end{tikzpicture}
    \caption{$\Set{a}$ is statically tenable, but $a\in \G^+$.}
    \label{fig:staticly tenable but conflict grounded}
\end{figure}
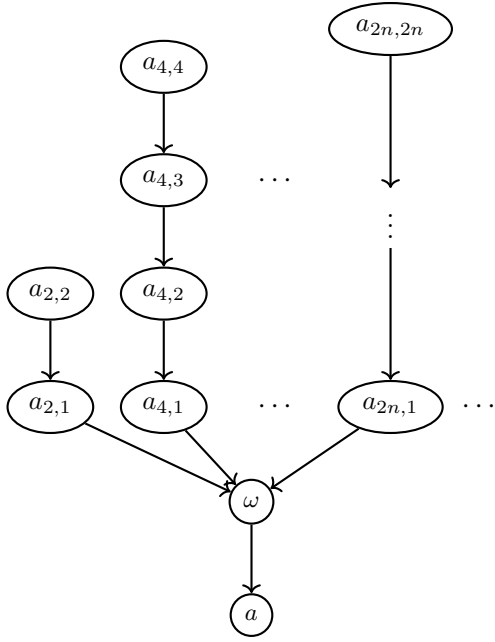
    \end{proof}

\begin{defn}
    The $n$-strong tenability game is the strong tenability game in which players may introduce at most $n$ new arguments in each move (except $\I$'s first move). A set is $n$-strongly tenable if $\I$ has a winning strategy beginning with $A$ in the $n$-strong tenability game.
\end{defn}

\StrongRobustness*
\begin{proof}
        Suppose $A$ is strongly tenable. We show that $A$ is $1$-strongly tenable by giving a strategy for $\I$ in the $1$-strong tenability game. Throughout the strategy, at each stage, we will consider a position in the strong tenability game for $A$. We call this the simulated game. To avoid confusion between the choices of moves in the (simulated) strong tenability game for $A$ and the real $1$-strong tenability game for $A$, we will refer to the simulated game by adding $\sl$ to each player. That is, we have now an auxiliary game, which is the strong tenability game for $A$ being played between the players $\I\sl$ and $\II\sl$.

        We now describe $\I$'s strategy in the main game, i.e., the 1-strong tenability game.
        $\I$ begins by playing $A$, to which $\II$ replies with some $b_0$. We let $\II\sl$ play $\Set{b_0}$ in the simulated strong tenablility game. In the simulated strong tenability game, following his winning strategy, $\I\sl$ replies with some set $A_1$ counterattacking $b_0$. Now $\I$ takes $a_1$ to be any argument in $A_1$ which counterattacks $b_0$ and plays $a_1$. In the $1$-strongly tenable game, we will write the played arguments as elements rather than sets.

        \begin{center}
            \begin{tabular}{c|c}
               $\I$  & $A$  \\\hline
               $\II$  & $b_0$
            \end{tabular}
            $\implies$
            \color{gray}
            \begin{tabular}{c|cc}
               $\I\sl$  & $A$  & $A_1$ \\\hline
               $\II\sl$  & $\Set{b_0}$ & 
            \end{tabular}
            \color{black}\\
            $\implies$
            \begin{tabular}{c|cc}
               $\I$  & $A$ & $a_1$\\\hline
               $\II$  & $b_0$ &
            \end{tabular}
        \end{center}

        Essentially, we follow this idea throughout the strategy. We model $\II$'s attacks as the plays of $\II\sl$ and use a single element from $\I\sl$'s response to counterattack this attack. Now we describe this more formally as a recursive algorithm. 
        
        Suppose we have arrived at the $i\th$ round of the dispute, assuming inductively that we have a position of the 1-strong tenability game and the simulated strong tenability game as follows:

        \begin{flushleft}
            \begin{tabular}{c|ccccc}
               $\I$\phantom{$\sl$}  & $A$ &$a_1$  &$\cdots$ & $a_{i-1}$ & $a_i$\\\hline
               $\II$\phantom{$\sl$}  & $b_0$ & $b_1$ & $\cdots$ & $b_{i-1}$ &  \\
            \end{tabular}\\[0.25cm]
            \color{gray}
            \begin{tabular}{c|cccc}
               $\I\sl$  & $A$  & $A_1$ & $\cdots$ & $A_{n_i}$\\\hline
               $\II\sl$  & $\Set{b_0}$ & $\cdots$ & $\Set{b_0,...,b_{i-1}}$\\
            \end{tabular}
        \end{flushleft}
        
        Further suppose that these satisfy the following properties:
        
        \begin{itemize}
            \item $a_i$ attacks $b_{i-1}$;
            \item $A\cup\Set{a_1,...,a_i}\subseteq A_{n_i}$;
            \item $A_{n_i}$ is always obtained by using the winning strategy 
            in the simulated $1$-strong tenability game.

        \end{itemize}
        
        Note that $n_i$ may be less than $i$, as our construction below may sometimes not need to play any moves in the simulated game. 
        If $\II$ has no further plays, then $\I$ wins and we are done. So assume that $\II$ plays the argument $b_i$ which attacks and is not counterattacked by $A\cup \Set{a_1,...,a_i}$. We must find an adequate defense.

        If there exists an element $a_{n+1}\in A_{n_i}$ which attacks $b_i$, then $\I$ plays this element. Note that the inductive hypotheses are maintained. This is the situation where we needn't play anything in the simulated game to find a proper response.

        If there is no such element in $A_{n_i}$, then $\Set{b_0,...,b_{i}}$ is a valid play against $A_{n_i}$ in the simulated game. We let $\II\sl$ play this set. Because we've assumed $A_{n_i}$ was obtained by repeatedly applying $\I\sl$'s winning strategy, the winning strategy will yield a further $A_{n_i+1}$ as cogent as $\Set{b_0,...,b_i}$. 
        Thus $A_{n_i+1}$ must contain an $a_{i+1}$ counterattacking $b_i$. This is $\I$'s next play. 
        The inductive hypothesis for the $(i+1)\th$ round of the game holds by construction.
        
        This argument then shows that $\I$ always has a legal move against any play of $\II$'s, hence always wins according to this strategy (either $\II$ runs out of plays, or the game is infinite). So $A$ is $1$-strongly tenable.

        For the converse implication, if $A$ is not strongly tenable, we show that $A$ is not 1-strongly tenable. We fix a strategy for $\II$ which wins in the strong tenability game. We find a strategy for $\II$ in the $1$-strong tenability game by simulating $\II\sl$ in the strong tenability game. In the first round of the game, $\I$ plays $A$ and $\II$ appeals to the simulated game to find $\II\sl$'s response $B_0$ to $A$. There must be some $b_0\in B_0$ which attacks $A$, so $\II$ plays $b_0$.

        \begin{center}
            \begin{tabular}{c|c}
               $\I$  & $A$  \\\hline
               $\II$  & 
            \end{tabular}
            $\implies$
            \color{gray}
            \begin{tabular}{c|c}
               $\I\sl$  & $A$   \\\hline
               $\II\sl$  & $B_0$
            \end{tabular}
            \color{black}\\ 
            $\implies$
            \begin{tabular}{c|cc}
               $\I$  & $A$ & \\\hline
               $\II$  & $b_0$ &
            \end{tabular}
        \end{center}

        Suppose we have arrived at the $i\th$ round of the dispute, assuming inductively that we have a position of the 1-strong tenability game and the simulated strong tenability game as follows:

        \begin{flushleft}
            \begin{tabular}{c|ccc}
               $\I$\phantom{$\sl$}  & $A$  &$\cdots$ & $a_i$\\\hline
               $\II$\phantom{$\sl$}  & $b_0$ & $\cdots$ & $b_i$\\
            \end{tabular}\\[0.25cm]
            \color{gray}
            \begin{tabular}{c|ccc}
               $\I\sl$  & $A$  & $\cdots$ & $A\cup\Set{a_0,...,a_j}$\\\hline
               $\II\sl$  & $B_0$ & $\cdots $ & $B_{n_i}$\\
            \end{tabular}
        \end{flushleft}

        Further suppose that these satisfy the following properties:
        \begin{itemize}
            \item $j\leq i$;
            \item $\Set{b_0,...,b_i}\not\preceq A\cup \Set{a_1,...,a_i}$;
            \item $\Set{b_1,...,b_i}\subseteq B_{n_i}$;
            \item For each $k<n_i$, $B_{k+1}$ is obtained by using the winning strategy 
            in the simulated strong tenability game.
        \end{itemize} 
        If $\I$ has no further plays, then $\II$ wins and we are done. So assume that $\I$ plays argument $a_{i+1}$ which attacks $b_{i}$. We must find an argument $b_{i+1}$ attacking $a_{i+1}$ which is not counterattacked.

        If some element $b_{i+1}\in B_{n_i}$ attacks  $A\cup\Set{a_0,...,a_{i+1}}$ and is not counterattacked, then we play this element. Otherwise,  
        $A\cup\Set{a_0,...,a_{i+1}}$ is a valid play against $B_{n_i}$ in the simulated game. We let $\I\sl$ plays this set. Because $B_{n_i}$ was assumed to be obtained by repeated applications of $\II\sl$'s winning strategy, a further application of the winning strategy yields a $B_{n_i+1}$ which contains a $b_{i+1}$ strictly attacking $A\cup\Set{a_0,...,a_{i}}$
        This is $\II$'s next play. 
        We have now moved to the following situation:

         \begin{flushleft}
            \begin{tabular}{c|cccc}
               $\I$\phantom{$\sl$}  & $A$  &$\cdots$ & $a_i$ & $a_{i+1}$\\\hline
               $\II$\phantom{$\sl$}  & $b_0$ & $\cdots$ & $b_i$ & $b_{i+1}$\\
            \end{tabular}\\[0.25cm]
            \color{gray}
            \setlength{\tabcolsep}{2.8pt}
            \begin{tabular}{c|cccc}
               $\I\sl$  & $A$  & $\cdots$ & $A\cup\Set{a_0,...,a_j}$ & $A\cup\Set{a_0,...,a_{i+1}}$\\\hline
               $\II\sl$  & $B_0$ & $\cdots $ & $B_{n_i}$ & $B_{n_i}+1$\\
            \end{tabular}
        \end{flushleft}

        The inductive hypothesis for the $(i+1)\th$ round of the game holds by construction.

        This argument then shows that $\II$ always has a legal move against any play of $\I$'s. Thus $\II$ cannot lose by failing to have any legal moves. Unlike the argument above, we also need to show that this winning strategy is finitary, i.e., leads to a concluded dispute.

        Suppose towards a contradiction that some sequence of moves for $\I$ would lead to the the main game having infinitely many moves. 
        Note that, at each stage where we do not extend the simulated game, $\II$ plays an element of $B_{n_i}$ which has not been previously played. Since each $B_{n_i}$ is finite, this must eventually stop, and the simulated game must continue. Thus we can conclude that the simulated game also has infinitely many moves. But this is contradictory to $\II\sl$ following the winning strategy for the simulated game. That is, since $\II\sl$ wins, the game must end in finitely many turns, which yields our contradiction.
\end{proof}

\begin{corollary}
    The following are equivalent for $A\subseteq\F$:
    \begin{enumerate}
        \item $A$ is strongly tenable; \label{statement a}
        \item $A$ is $1-$strongly tenable; \label{statement b}
        \item $A$ is $n-$strongly tenable for some $n$;\label{statement c}
        \item $A$ is $n-$strongly tenable for all $n$;\label{statement d}
    \end{enumerate}
\end{corollary}
    \begin{proof}
        Statements (\ref{statement a}) and (\ref{statement b}) are equivalent by the statement of the previous theorem. Fix any $n>1$. To see the equivalence of $A$ being $1$-strongly tenable and $A$ being $n$-strongly tenable, apply the same proof strategy above, but by simulating the winning strategies in the $n$-strong tenability game instead of the strong tenability game. This then shows the equivalence of (\ref{statement b}), (\ref{statement c}), and (\ref{statement d}).
    \end{proof}

\Complexity*
\begin{figure*}[b]
    \centering
    \begin{tikzpicture}
        [thick,node distance =15mm, on grid, 
        arg/.style={draw, circle, minimum size =7mm, font=\small},
        ] 
            \node at(-2.5,0) (null){};
            \node[arg] at (0,0) (a) {$a$};
            \ForallWidget{1}{a}{a}{a}
            \SwitchWidgetWithLabel{center1}{x1}{notx1}{2}
            \ForallWidget{2}{switch}{switch}{switch}
            \ExistsWidget{3}{center2}{x2}{notx2}
            \SwitchWidget{center3}{x3}{notx3}
            \SwitchWidgetWithLabel{center3}{x3}{notx3}{4}
            \PhiWidget{switch}{switch}{switch}
            \node[ov, red!75!black] (c2) [right of =phi]{$\neg u_2$};
            \node[ov,cyan!90!black] (c1) [above=10mm of c2] {$\neg u_1$};
            \node[ov,violet] (c3) [below=10mm of c2] {$\neg u_3$};
            \draw[->] (c1) to (phi);
            \draw[->] (c2) to (phi);
            \draw[->] (c3) to (phi);  

            \draw[->,cyan!90!black] (x1) to [out=40, in=160] (c1);
            \draw[->, cyan!90!black] (notx3) to [out=60, in=180] (c1);
        
            \draw[->,red!75!black] (notx2) to [out=330, in=225] (c2);
            \draw[->,red!75!black] (x3) to [out=20, in=135] (c2);

            \draw[->,violet] (x2) to [out=60, in = 0, looseness =2](c3);            
    \end{tikzpicture}
    \caption{An Example $\mathcal H$ which encodes the QBF $\forall x_1\forall x_2\exists x_3[(x_1\vee \neg x_3) \wedge(\neg x_2\vee x_3) \wedge x_2] $}
    \label{fig: example H}
\end{figure*}
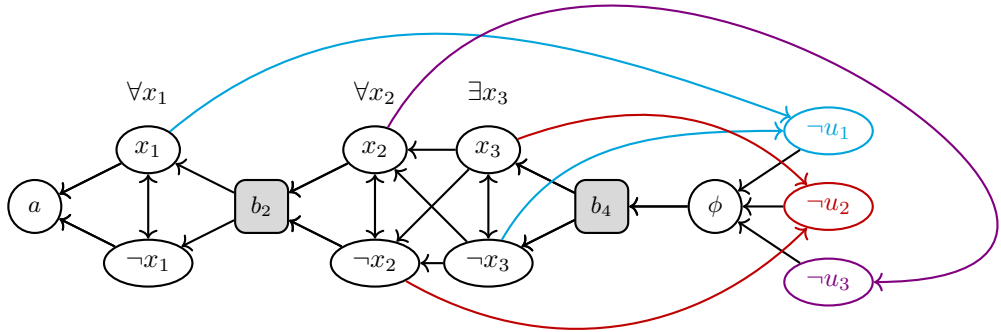
\begin{proof}
    We begin by showing that the problem of determining if $A\subseteq F$ is strongly tenable is PSPACE-complete. Let $n$ be the number of arguments in $\F$. We first note that the game-tree for the (strong/-) tenability game has depth at most $n$ and the branching is bounded by $2^n$, since each move is a subset of $\F$. From this, we conclude that the decision problem is in PSPACE. This follows from the characterization $\text{PSPACE} = \text{APTIME}$ \cite{APTIME}. 
    Since the game ends in $n$ rounds, and each move can be specified using polynomial space (i.e., the branching factor is at most single-exponential), the game can be decided by an alternating computation of polynomial depth, thus is in APTIME.
    
    To show hardness, we reduce the problem QSAT of determining the truth of a quantified boolean formula (QBF) to whether an element in an AF is (strongly/-) tenable.

    Let $\rho := Q_1 x_1 Q_2 x_2 \dots Q_n x_n \phi(\bar{x})$ be a quantified boolean formula. Let $\phi$ be in conjunctive normal form $\bigwedge_i u_i$ where each $u_i$ is a disjunction of propositional variables ($x_j$'s) and negations of propositional variables ($\neg x_j$'s). 
    
    We construct an AF from this formula inductively. We begin with the AF $\F_0$ with one argument $a$ and no attacks. For the sake of giving an inductive definition, we also call the argument $a$ by the names $x_0$ and $\neg x_0$ and we define $Q_0=\exists$. 
    Given $\F_k$ for $k\leq n$, we construct $\F_{k+1}$ as follows: If $Q_k=Q_{k+1}$, we form $\F_{k+1}$ by adding three new elements to $\F_k$, which we call $b_{k+1}$, $x_{k+1}$, and $\neg x_{k+1}$. We let $x_{k+1}$ and $\neg x_{k+1}$ both attack $b_{k+1}$, we let $b_{k+1}$ attack $x_{k}$ and $\neg x_k$, and we let $x_k$ and $\neg x_k$ attack each other. See Figure \ref{fig:consecutive quantifiers}.

    On the other hand, if $Q_k\neq Q_{k+1}$, then we form $\F_{k+1}$ by adding only two new elements to $\F_k$, which we call $x_{k+1}$, and $\neg x_{k+1}$. We let $x_{k+1}$ and $\neg x_{k+1}$ both attack each of $x_{k}$ and $\neg x_k$, and we let $x_k$ and $\neg x_k$ attack each other. See Figure \ref{fig:alternating quantifiers}.

\begin{figure}[h!]
    \centering

    \begin{tikzpicture}
        [thick,node distance =12mm, on grid, 
        arg/.style={draw, circle,font=\tiny, minimum size =7mm},
        ]

            \node at (0,0) (a) {$\cdots$};
            \ForallWidget{i}{a}{a}{a}
            \SwitchWidgetWithLabel{centeri}{xi}{notxi}{i+1}
            \ForallWidget{i+1}{switch}{switch}{switch}
            \node (dots) [right of = centeri+1] {$\cdots$};

            \node at (0,-3.5) (a) {$\cdots$};
            \ExistsWidget{i}{a}{a}{a}
            \SwitchWidgetWithLabel{centeri}{xi}{notxi}{i+1}
            \ExistsWidget{i+1}{switch}{switch}{switch}
            \node (dots) [right of = centeri+1] {$\cdots$};

    \end{tikzpicture}
    
    \caption{Consecutive quantifiers of the same kind are joined by a \textit{switch node}.}
    \label{fig:consecutive quantifiers}

    \begin{tikzpicture}
        [thick,node distance =20mm, on grid, 
        arg/.style={draw, circle,font=\tiny, minimum size =7mm},
        ] 

            \node at (0,0) (a) {$\cdots$};
            \ForallWidget{i}{a}{a}{a}
            \ExistsWidget{i+1}{centeri}{xi}{notxi}
            \node (dots) [right =10mm of centeri+1] {$\cdots$};

            \node at (0,-3.5) (a) {$\cdots$};
            \ExistsWidget{i}{a}{a}{a}
            \ForallWidget{i+1}{centeri}{xi}{notxi}
            \node (dots) [right =10mm of centeri+1] {$\cdots$};

    \end{tikzpicture}
    \caption{Alternating quantifiers are joined directly.}
    \label{fig:alternating quantifiers}
\end{figure}

    Following this definition, we get an AF $\F_n$. If $Q_n=\exists$, we then define $\F_{n+1}$ by adding two new elements to $\F_n$, which we call $b_{n+1}$ and $\phi$ and letting $\phi$ attack $b_{n+1}$ and $b_{n+1}$ attack both $x_n$ and $\neg x_n$. 

    If $Q_n=\forall$, we then define $\F_{n+1}$ by adding one new element to $\F_n$, which we call $\phi$ and letting $\phi$ attack both $x_n$ and $\neg x_n$.

    Finally, we define $\mathcal{H}$ by adding arguments which we call $\neg u_i$, one for each of the conjuncts $u_i$, and allowing each $\neg u_i$ to attack $\phi$. We also let each $x_i$ attack $\neg u_i$ if $x_i$ is a disjunct in $u_i$, and we let each $\neg x_i$ attack $\neg u_i$ if $\neg x_i$ is a disjunct in $u_i$. See figure \ref{fig: example H} for a sample of the fully constructed $\mathcal H$ encoding a QBF. 

    We next show that if $\rho$ is true, then $a$ is strongly tenable in $\mathcal H$. Suppose that $\rho$ is true. We give a strategy for $\I$ in the strong tenability game for $a$ in $\mathcal H$. By Theorem \ref{thm: strong robustness}, we may assume that all moves by $\II$ introduce only a single new element.
    
    If $Q_1$ is a $\forall$, then $\II$ must play either $x_1$ or $\neg x_1$. If $Q_1=\exists$ then we can assume that $\II$ plays $\{b_1\}$ and $\I$ must respond by playing either $\{x_1\}$ or $\{\neg x_1\}$. $\I$ makes this choice based on which choice makes $\rho_1(x_1):=Q_2 x_2 \dots Q_n x_n \phi(\bar{x})$ true. We now consider later quantifiers $Q_{k}$: If $Q_{k}=Q_{k-1}$, then whichever player chose $\{x_{k-1}\}$ or $\{\neg x_{k-1}\}$, the other player must choose $b_k$, and this same player must choose between $\{x_{k}\}$ and $\{\neg x_k\}$ for their next play. If $Q_k\neq Q_{k-1}$, then whichever player chose $\{x_{k-1}\}$ or $\{\neg x_{k-1}\}$, the opposite player must choose between $\{x_k\}$ and $\{\neg x_k\}$. By induction, $\I$ chooses between $x_k$ and $\neg x_k$ if and only if the quantifier $Q_k$ is $\exists$. In this case, $\I$ chooses $x_k$ or $\neg x_k$ so as to maintain the truth of the formula $\rho_k(x_1,\dots , x_k):= Q_{k+1}x_{k+1}\dots Q_n x_n \phi(\bar{x})$. This determines how the game must be played through the AF $\F_n$. Note that by definition, $\I$ will play the argument $\phi$ in $\F_{n+1}$. Since she has maintained the condition that $\rho_k$ is true for each $k$, we have that $\phi(\bar{x})$ is true. In particular, for any choice of $\neg u_i$ which $\II$ might play, there is already some played (by either $\I$ or $\II$) element $x_k$ or $\neg x_k$ which attacks this $\neg u_i$. In particular, player $\II$ has no legal moves once $\I$ plays $\phi$, and thus $\I$ wins the game.

    Now we suppose that $\rho$ is false. Then the symmetric strategy for $\II$ works: Namely, he chooses $x_k$ or $\neg x_k$ whenever $Q_k=\forall$ in a way which maintains the hypothesis that $\rho_k$ is false. The game still proceeds through $\F_n$ and then $\I$ must play the argument $\phi$. At this point, some $u_i$ must be false, since $\rho_n=\phi(\bar{x})$ is false. So, $\II$ plays the argument $\neg u_i$ and $\I$ has no response. We thus see that $\rho$ is true if and only if $a$ is strongly tenable in $\mathcal H$. 
    
    Further, we now argue that if $\rho$ is false, then $a$ is not tenable in $\mathcal H$. The strategy is essentially as above, except in the tenability game we may not assume that $\I$ plays only single element extensions. At turn $s$, $\I$ plays $x_{i}$ or $\neg x_i$ if 
    \begin{itemize}
        \item $Q_i=\forall$;
        \item For each $j<i$ with $Q_j=\exists$, $\I$ has already played $x_j$ or $\neg x_j$;
        \item For each $b_j$ with $j\leq i$ and $Q_j=\forall$, $\I$ has already played $b_j$.
    \end{itemize} 

    $\II$ chooses between $x_i$ and $\neg x_i$ to ensure that $\rho_i$ is false. Since each $x_j$ or $\neg x_j$ with $Q_j=\exists$ is chosen before $x_i$ is chosen, $\II$ can maintain this for each $i$ when he chooses between $x_i$ and $\neg x_i$.

    On turn 1, $\II$ can also play all $b_j$ with $Q_j=\exists$ or $j=n+1$. By this strategy, as above, $\I$ must eventually play each $b_i$ with $Q_i=\forall$, either $x_j$ or $\neg x_j$ for each $j$ so $Q_j=\exists$, and $\phi$. Since $\phi$ is false with the choices made above, once $\I$ plays all of these elements, $\II$ can play some $\neg u_i$ so that $u_i$ is false, and $\I$ then has no response.

    We conclude that $\rho$ being true if and only if $a\in \mathcal H$ is strongly tenable if and only if $a\in \mathcal H$ is tenable. This shows the PSPACE-hardness of both tenability and strong tenability.

    We now shift to showing that the decision problem for static tenability is $\Pi_2^P$-complete. Note that its form is inherently $\Pi_2^P$: For any $B$, there exists a $C$ so that if $B$ is conflict-free, then $C$ is conflict-free and is as cogent as $B$. This is Co-NP using an NP oracle, so we need only show $\Pi_2^P$-hardness. We do this by reducing QBF$^2_\forall$ to the decision problem for static tenability. Let $\rho:= \forall x_1\dots \forall x_n \exists y_1 \dots \exists y_m \phi(\bar{x},\bar{y})$ be a formula in QBF$^2_\forall$. Let $\mathcal H$ be the AF constructed above for $\rho$. We now claim that $a\in \mathcal H$ is statically tenable if and only if $\rho$ is true.

    Suppose that $\rho$ is true. Then by the above argument, $a$ is strongly tenable in $\mathcal H$, thus $a$ is statically tenable in $\mathcal H$.

    Next, suppose that $\rho$ is false. Let $\tau$ be a truth assignment of $\bar{x}$ which makes $\rho'(\bar{x}):=\exists \bar{y} \phi(\bar{x},\bar{y})$ false. Let $B_0=\{ x_i \mid \tau(x_i)=\text{true}\}\cup \{\neg x_i \mid \tau(x_i)=\text{false}\}\cup \{b_k \mid \vert k>n\}$. Let $U=\{\neg u_i \mid B_0\cup \{\neg u_i\}\text{ is conflict-free}\}$. Finally, let $B=B_0\cup U$. We argue that there is no set $C\supseteq A$ as cogent as $B$. Suppose towards a contradiction that such a $C$ existed. Then $C$ must contain $b_i$ for $i\leq n$, in order to counterattack the defeaters of $a$. Similarly, it must contain exactly one of $\{y_1, \neg y_1\}$ in order to counterattack the defeaters of $b_{n}$. Then it must contain exactly one of $\{y_{i},\neg y_i\}$ for each $i$ in order to defend $y_{i-1}$ or $\neg y_{i-1}$ from $b_{n+i}$. Finally, it must contain $\phi$ in order to defend $y_{m}$ or $\neg y_{m}$ from $b_{n+m+1}$. Let $\sigma$ be the truth assignment for $\bar{y}$ given by these choices. Then there is some $u_i$ so that $\neg u_i$ is true in the assignment $\tau \cup \sigma$. Since this $\neg u_i$ is not contradicted by $\tau$, $\neg u_i\in B$. This attack of $\phi$ is not counterattacked by any $y_j\in C$ or $\neg y_j\in C$ since $\neg u_i$ is true with the assignment $\tau\cup \sigma$. It is not counterattacked by any $x_i$ or $\neg x_i$ in $C$ since these are all in conflict with the $b_i$'s which are in $C$. Thus this attack is not counterattacked, contradicting that $C$ is as cogent as $B$. Thus $\rho$ is true if and only if $a\in \mathcal H$ is statically tenable, showing the $\Pi_2^P$-completeness of the decision problem for static tenability.
\end{proof}

\subsection{Omitted proofs in section \ref{subsec: tenability infinite setting}}
We now give a set theoretic analysis of the notions of (strong/-) tenability in the infinite setting.

The following definition captures which positions are winning positions (via strategies that end in finite time) for each player. Below, $\S(\tau)=1$ should be understood as ``$\tau$ is a position from which $\I$ can force a finite win'' and $\S(\tau)=2$ should be understood as ``$\tau$ is a position from which $\II$ can force a finite win''. In the analysis, it is critical to understand the complexity of the definition of this function.

\begin{defn}
    Let $T$ be a \textit{game tree} for the tenability (or strong tenability) game on $\F$. That is, $T$ is the set of all (-/strong) tenability disputes $(X_0,X_1,\dots X_n)$ in the game with the relation $\preceq$ defined by prefix. That is, $(X_0)\preceq (X_0,X_1)\preceq (X_0, X_1, X_2)\dots$ for any play in the game. Paths through the tree (allowing for a path to end in a leaf) represents all possible plays of the game (including those with infinitely many turns).
\end{defn}

    We define a particular \textit{strategy labeling} $\S:T\to \{1,2\}$, which is obtained by labeling all leaf nodes (corresponding to concluded disputes) `1' or `2' according to the player who  wins that dispute. That is, odd length leaves, corresponding to concluded disputes won by $\I$ are labeled `1' and even length leaves, corresponding to concluded disputes won by $\II$ are labeled `2'. Then we  propagate the labels downward through the tree, describing the positions that are winning positions for $\I$ or $\II$. Precisely, if $\sigma$ is of odd length (a $\I$ move), then it gets labeled `1' if every immediate successor (possible $\II$ moves) has already been labeled `1'; this indicates that $\I$ can play $\sigma$ to advance towards victory. If $\sigma$ is of odd length (a $\I$ move), and there is some immediate successor $\tau$ (a $\II$ move) that has already been labeled `2', then $\sigma$ gets labeled `2', as $\II$ can play $\tau$ to advance towards victory; $\I$ should not play $\sigma$ if possible. The reasoning is similar when $\sigma$ is of even length.  This propagation may require infinitely many steps, therefore we take unions at limit stages.
\begin{align*}
    \S_0(\sigma)&:=\begin{cases}
        1 & |\sigma|=2n+1, \text{ $\sigma$ a leaf}\\
        2 & |\sigma|=2n, \text{ $\sigma$ a leaf}
        \end{cases}\\[0.25cm]
    \S_{\alpha+1}(\sigma)&:=\\&\begin{cases}\small
        \S_\alpha (\sigma) & \sigma \in \dom(\S_\alpha)\\
        1 & |\sigma|=2n+1 \text{ and }\\ &(\forall \tau \succ \sigma)[|\tau|=|\sigma|+1 \to \S_\alpha(\tau)=1]\\
        1 & |\sigma|=2n \text{ and }\\ &(\exists \tau \succ \sigma)[|\tau|=|\sigma|+1 \wedge \S_\alpha(\tau)=1]\\
        2 & |\sigma|=2n \text{ and }\\ &(\forall \tau \succ \sigma)[|\tau|=|\sigma|+1 \to \S_\alpha(\tau)=2]\\
        2 & |\sigma|=2n+1 \text{ and }\\ &(\exists \tau \succ \sigma)[|\tau|=|\sigma|+1 \wedge \S_\alpha(\tau)=2]\\
        \end{cases}\\[0.25cm]
    \S_\lambda &:=\bigcup_{\beta<\lambda}\S_\beta \text{ for } \lambda \text{ limit}\\
    \S & :=\bigcup_{\text{Ordinals }\alpha} \S_\alpha
\end{align*}

\begin{lemma}\label{lem: conditiosn for winning strats} ($ZFC$)
    If $S(\sigma)=1$ then $\I$ has a strategy from the game position $\sigma$ which wins in finite time. If $\S(\sigma)=2$ then $\II$ has a winning strategy from the game position $\sigma$.
    \begin{proof}
        By induction on the least $\alpha$ for which $\S_\alpha(\sigma)=1$ (or $2$). This is clear for $\alpha=0$. Similarly, no new nodes are labeled at limit stages. So assume it holds inductively for $\S_\alpha$, and that $\S_{\alpha+1}$ is the first to label $\sigma$ as $1$. If $\sigma$ is of odd length, then all of its successors $\tau$ must be labeled $1$ by $\S_\alpha$. By inductive hypothesis, $\I$ has a strategy winning in finitely many steps from any $\tau$ that $\II$ may choose next. If $\sigma$ is of even length ($\II$ has made the last move), then it has some successor $\tau$ which is labeled $1$ by $\S_\alpha$. Player $1$ then chooses (Axiom of Choice) one such $\tau$, and continues the strategy (given by inductive hypothesis) from $\tau$. Follow a symmetric argument for $\S_\alpha(\sigma)=2$.
    \end{proof}
\end{lemma}

\begin{lemma}\label{lem: pro winning strat iff sgneq2} ($ZFC$)
    $\I$ has a winning strategy from $\sigma$ if and only if $\S(\sigma)\neq 2$. (i.e., $\S(\sigma)=1$ or $\sigma \notin \dom(\S)$)
    \begin{proof}
        The previous lemma shows what happens when $\S(\sigma)=1$, so we attend to when $\sigma$ is not decided by $\S$. The idea is simple: when $\S$ is undecided about $\sigma$, it corresponds to neither player having a finitary winning strategy from $\sigma$. Because $\I$ wins infinite plays of the game, this corresponds to a win for $\I$.

        More precisely, the strategy goes as follows: On his turn, from position $\rho$, $\I$ moves to any $\tau$ so that $\S(\tau)\neq 2$ (if possible). We now verify by induction that regardless of $\II$'s plays, the game never passes through a position $\tau$ so that $\S(\tau)=2$. We begin in position $\sigma$ and we have assumed that $\S(\sigma)\neq 2$. From a node $\rho$ so $\S(\rho)\neq 2$, $\I$ always chooses to move to a new node $\tau$ so that $\S(\tau)\neq 2$. From a node $\rho$ so that $\S(\rho)\neq 2$, by definition, $\II$ cannot move to any node $\tau$ so that $\S(\tau)= 2$ (otherwise, $\rho$ would also be labeled 2). Thus the game cannot end in an even length leaf node, i.e., in a victory for $\II$, since these are all labeled 2.
    \end{proof}
\end{lemma}

\Determinacy*
\begin{proof}
    Apply Lemma \ref{lem: pro winning strat iff sgneq2} to $\sigma=\Set{A}$.
    Observe that any transitive models containing both $A$ and $\F$ have the same finite subsets of $\F$ hence have the same game tree above $A$. $\S$ is defined by transfinite recursion using only $\F$ as a parameter, thus the statement that $\S(\Set{A})=2$ is $\Delta_1(A,\F)$, hence is absolute between transitive models of ZFC. (For more on absoluteness, see, e.g., \cite[Chapter 13]{Jech2003}.)

    Because Lemmas \ref{lem: conditiosn for winning strats} and \ref{lem: pro winning strat iff sgneq2} are theorems of ZFC, if $\S(\Set{A})=2$, both models contain the winning strategy for $\II$; if $\S(\Set{A})\neq 2$, both models contain the winning strategy for $\I$.

    In the case of static tenability, if $A$ is statically tenable, then for every conflict-free set $B$, there is a set $C\supseteq A$ as cogent as $B$. Note that $C\setminus A$ can always be taken to be finite; we need only find a minimal number ($\leq|B|$-many) of elements and join them to $A$ to counterattack those attackers. Thus any two transitive models of $ZFC$ containing $(A,\F)$ will agree on the finite subsets of $\F$ and hence on the static tenability of $A$. 

\end{proof}

\subsection{Omitted proofs in Section \ref{sec:other semantics}}
\WeakSemanticsZoo*

    \begin{figure}[h]
        \centering
        \tikzfinitezoo
        \repeatcaption{fig: weak semantics summary-diagram}{The relations among weak semantics in finite argumentation frameworks.
        }
    \end{figure}

We establish this theorem in the following discussion.

We will need to work with each of these semantics, so their definitions and some known results are collected in the next subsection:

\subsubsection{Definitions and basic results about the weak semantics}

We gather here the definitions of the semantics which we consider. Strong tenability, tenability, and static tenability are defined in Section \ref{sect: notions of tenability} above.

\begin{defn}\label{def:atleast}
For $E,D\subseteq\F$, we say that $E$ is \textit{at least as cogent as} $D$, written $E\succeq D$,
if $E$ is \textit{admissible} in the restriction $\F\!\upharpoonright_{E\cup D}$.
We write $E\succ D$ if $E\succeq D$ and  $D\not\succeq E$.
\end{defn}

\begin{defn}\label{def:cogent}
$E\subseteq\F$ is \textit{cogent}, written $E\in\cog(\AF)$, if
\[
(\forall D\subseteq\F) (D\succeq E \Rightarrow E\succeq D).
\]
\end{defn}

Hence, a cogent extension is undominated with respect to $\succeq$.

The following known characterization is the one we use for comparisons.

\begin{proposition}[
{\cite[Theorem 4.4]{BlumelUlbricht}}]\label{prop:cog-char}
For $E\subseteq\F$,
\begin{align*}
    E&\in\mathrm{cog}(\F)\\
        &\iff\\
        E&\in\cf(\F)  \text{ and}\\
        (\forall a\in \F)[(a\not\att &a \wedge a\att E)\to a\in E^+].
\end{align*}
\end{proposition}

\begin{defn}\label{def:cyclic}
$E\subseteq\F$ is \textit{cyclically cogent}, written $E\in\cycog(\F)$, if for every $D\subseteq\F$
at least one of the following holds:
\begin{enumerate}[label=(\arabic*),leftmargin=*]
\item $D\not\succeq E$;
\item $E\succeq D$;
\item there exists a chain $D=D_1\prec D_2\prec \dots \prec D_n=E$.
\end{enumerate}
\end{defn}

\def\sust{\text{sust}}
\def\tol{\text{tol}}
\def\lax{\text{lax}}

\begin{defn}
    Let $\F=\AF$ be an argumentation framework and $A\subseteq F$. Then $\F^A$ is the reduct of $\F$ to $F\smallsetminus (A\cup A^+)$.

    Let $\F$ be an argumentation framework. A set $A\subseteq \F$ is weakly admissible if it is conflict-free and for every attacker $y$ of $A$, $y$ belongs to no weakly admissible extension of $\F^A$. Denote the set of weakly admissible extensions of $\F$ by  $\wadm(\F)$.
\end{defn}

\begin{defn}
    A \textit{weakly complete labeling} of $\AF$ is one where for every $a\in \F$, the following holds:
    \begin{itemize}
        \item If $a$ is labeled \texttt{in} then there are no attackers of $a$ labeled \texttt{in}.
        \item If $a$ is labeled \texttt{out} then it has at least one attacker labeled \texttt{in}.
        \item If $a$ is labeled \texttt{undec}, then it has at least one attacker labeled \texttt{undec}, and it has no attackers labeled \texttt{in}.
    \end{itemize}

    An extension is \textit{weakly complete} if it is the set of \texttt{in}-labeled arguments in a weakly complete labeling.
\end{defn}

This Lemma was proved for finite AFs in \cite[Theorem 4.1]{DondioLongo2021}.

\begin{lemma}\label{lem:wcomp contains grounded}
    If $A$ is weakly complete, then $A$ contains the grounded extension.
    \begin{proof}
        Observe that if $A$ defends $x$, then every weakly complete labeling which labels every element of $A$ with \tin must also label $x$ with \tin. Therefore, since every element of $\emptyset$ is labeled \tin, every element of the least fixed point of the defense operator, i.e., the grounded extension, must also be labeled \tin.
    \end{proof}
\end{lemma}

\def\Gub{\G_{ub}}
\def\Lgub{\L_{\Gub}}
\def\ubg{ub-\G}

\subsubsection{Implications} We now verify all the arrows and dashed arrows of Figure \ref{fig: weak semantics summary-diagram}.

\begin{proposition}
    Cogent implies cyclically cogent.
\end{proposition}

\begin{proof}
Let $E\in\cog(\F)$ and fix $D\subseteq\F$.
If $D\not\succeq E$ then Definition~\ref{def:cyclic}(1) holds.
If $D\succeq E$, then by cogency $E\succeq D$, so (2) holds. Hence $E\in\cycog(\F)$.
\end{proof}

\begin{proposition}
    Cogent implies strongly tenable.
\end{proposition}
\begin{proof}
    Using Proposition \ref{prop:cog-char}, we see that if $A$ is cogent, then $\II$ has no move from the initial position $(A)$. In particular, every conflict-free attacker of $A$ is already counterattacked by $A$.
\end{proof}

We established the implications from Strong Tenability to Tenability to Static Tenability in Proposition \ref{Prop:TenabilityImplications}.

We note that the following Proposition has appeared recently in \cite{BodanzaCogentImpliesWADM}. We include an argument for completeness.
\begin{proposition}
    Cogent implies weak admissibility.
\end{proposition}
\begin{proof}
    Let $A$ be cogent.
    By Proposition \ref{prop:cog-char}, $A$ counterattacks all attacks from non-self-attackers. Thus in the framework $\F^A$, every attacker $y$ of $A$ is a self-attacker. In particular, it is not contained in any conflict-free set, so it cannot belong to a weakly admissible set of $\F^A$. 
\end{proof}

\begin{theorem}\label{thm:tenable implies wc}
    If $A$ is tenable, then $A$ is contained in a weakly complete extension.
\end{theorem}
\begin{proof}
    Suppose that $A$ is not contained in a weakly complete extension. Let $\mathcal H$ be the restriction of $\F=\AF$ to the set of arguments $F\smallsetminus (A\cup A^+)$. Let $\G$ be the grounded extension of $\mathcal H$. We propose the labeling which labels $A\cup \G$ as $\texttt{in}$, $(A\cup \G)^+$ as \texttt{out}, and all other elements as $\texttt{undec}$. Note that if $b$ is labeled $\texttt{out}$ then it has an attacker labeled \texttt{in}. Notice also that if $a$ is labeled $\texttt{undec}$, then it has no attacker labeled $\texttt{in}$. Further, it must have some attacker outside of $(A\cup \G)^+$, as otherwise it would be defended in $\mathcal H$ by $\G$, thus would be in $\G$. Thus any element labeled \texttt{undec} has an attacker labeled \texttt{undec}. Finally, if $a\in \G$ then it has no attacker in $A\cup \G$ since $\G$ is conflict-free and $A^+$ is disjoint from $\mathcal H$. Thus either $A\cup \G$ is weakly complete, contrary to the hypothesis, or there is some argument $a\in A$ which is attacked by an argument $g\in \G$. We now employ the proof-strategy of Lemma \ref{lem:tenable implies cf with grounded} to give a winning strategy for $\II$ to show that $A$ is not tenable. 

     $\II$ begins by playing $g\in \G$ which attacks $a$.  $\I$ must play some argument $a_2$ which attacks $g$.  $\II$ plays some $g_2\in \G$ which enters $\G$ before $g_1$ which attacks $a_2$. Since $\II$ always plays in $\G$, he remains conflict-free. Since $A$ does not attack $\G$, and $\I$ cannot play an element of $A^+$, $\I$ must always play in $\mathcal H$, and $\II$'s strategy produces a descending sequence of ordinals, thus the game must end, and $\II$ must win.  
\end{proof}

\begin{theorem}
    If $\F$ is finite and $A$ is statically tenable, then $A$ is contained in a weakly complete extension.
\end{theorem}
\begin{proof}
    Suppose that $\F$ is finite and $A$ is not contained in a weakly complete extension. We again define $\mathcal H$ and $\G$ as in Theorem \ref{thm:tenable implies wc}. Again, if $A\cup \G$ is not weakly complete, then some element of $\G$ attacks an element of $A$. Let $B=\G$ and we will show that this $B$ witnesses that $A$ is not statically tenable. Suppose towards a contradiction that $C\supseteq A$ was defended against $\G$. Let $n$ be least so that some member $g$ of $\G_n$ attacks an argument in $C$. Such an $n$ exists since $\G$ has an attacker of $A$. Then since $C$ is as cogent as $\G$, it must have an element $c$ which counterattacks $g$. Note that this $c$ cannot be in $A$ since $A$ does not attack $\G$. It cannot be in $A^+$ since $C$ is conflict-free. Thus $c\in \mathcal H$. Thus there is some element of $\G_{n-1}$ which attacks $c$ contradicting the minimality of $n$.
\end{proof}

\def\lab{\mathcal{L}}

\begin{lemma}\label{lem: weak adm defense}
    If $A$ is weakly admissible and $A$ (classically) defends $X$, then $A\cup X$ is weakly admissible.
\end{lemma}
\begin{proof}
    Observe that $A\cup X$ is conflict-free: if (a) $A\att x\in X$, then $A\att A$ by virtue of defending $X$; if (b) $x\in X$ attacks $a$ then either $A\att x$ (which is impossible by (a)) or (c) $x\notin \cred_{\wadm}(\F^A)$. The last case, (c) is false as $X\in \wadm(\F^A)$ by virtue of having no attackers in $\wadm(\F^A)$. 

    Suppose $y\att A\cup X$ is not attacked by $A\cup X$. We must show that $y$ is contained in no weakly admissible extension of $\F^{A\cup X}$. Suppose towards a contradiction that there is some such weak admissible extension $B\ni y$ of $\F^{A\cup X}$. 

    Note that $y\att A$, since $A$ defends $X$ and $y$ is not attacked by $A$.
    We wish to prove that $B\cup X$ is weakly admissible in $\F^A$ showing that $y\in \cred_{\wadm}(\F^A)$, contradicting the hypothesis that $A$ is weakly admissible.

    Observe that every attack from $b\in B$ to $A\cup X$ is in fact an attack $b\att A$; otherwise $b$ is simply not included in $F^{A\cup X}$. Hence the set $B\cup X$ is conflict-free. Suppose there is some $C\in \wadm((F^A)^{B\cup X})$ which attacks $B\cup X$. Notice that $(\F^A)^{B\cup X}=\F^{A\cup B\cup X}=(\F^{A\cup X})^{B}$. By the weak admissibility of $B$ in $\F^{A\cup X}$, $C$ cannot be weakly admissible in $(\F^A)^{B\cup X}$. Thus $B\cup X$ is weakly admissible in $F^A$ and attacks $A$, contradicting the weak admissibility of $A$.
\end{proof}

\begin{theorem}\label{thm: wadm contained in wcomp}
    If $\F$ is finite and $A\subseteq\F$ is weakly admissible, then $A$ is contained in a weakly complete extension.
    \begin{proof}
        To produce a weakly complete labeling $\L$ of $F$, we define $A_0:= A$, and $A_{i+1}=A_i\cup \{a: A_i \text{ defends } a\} $. By Lemma \ref{lem: weak adm defense}, this sequence of sets is well-defined, and $\bigcup A_i$ is conflict-free and weakly admissible. Note that $A_n=A_{n+1}$ for some $n$. The sequence of sets induces a natural labeling:
        \begin{align*}
            \L(x)=\begin{cases}
                \tin & \text{if } x\in \bigcup A_i\\
                \tout & \text{if } x\in \bigcup A^+_i\\
                \tun &\text{otherwise}
            \end{cases}
        \end{align*}
        We confirm this labeling is weakly complete. If $\L(a)=\tin$, then $a\in A_j$ for some $j$, so all $b\in A_j^+\subseteq \bigcup A^+_i$ are labeled $\tout$. Similarly, if $\L(b)=\tout$ then it has an attacker labeled $\tin$. If $L(x)= \tun$, then $x$ is not attacked by any $a\in \bigcup A_i$. It must also be the case that there is some $y\att x$ such that $\L(y)=\tun$, otherwise $x$ is defended by some $A_j$ by virtue of all of its attackers belonging to $\bigcup A^+_j$. But this would make $L(y)=\tin$.
        Therefore, $\L$ is a weakly complete labeling of $\F$ and the set $A'=\{a\in \F: \L(a)=\tin\}$ is weakly complete, and contains $A$.
    \end{proof}
\end{theorem}

We note a certain surprising symmetry: If you close an admissible extension under the defense operator, you get a complete extension. The previous Theorem shows that if you close a weakly admissible extension (in the sense of Baumann, Brewka, and Ulbricht) under the defense operator, you get a weakly complete extension (in the sense of Dondio and Longo).

\subsubsection{Non-implications}

In this section, we show all the necessary non-implications to conclude Theorem \ref{thm: weak semantics zoo}. To do this, we refer to the collection of example argumentation frameworks from Table \ref{tab: cred acceptance}. For each of these argumentation frameworks and various semantics $\sigma$, Table \ref{tab: cred acceptance} clarifies whether or not the argument $a$ is credibly accepted for $\sigma$. In each case, it is a matter of checking the definition to establish these conclusions. We do not provide a proof for each example and semantics.

\begin{table*}[!t] 
    \centering

    {\tikzset{
    every picture/.style={
      execute at end picture={
        \path (current bounding box.south west) +(0,-0.1) (current bounding box.north east)+(0,0.1);
        }
      }
    }
    \begin{tabular}{|cc||c||c|c|c|c|c|c|c|}
        \hline
        
        & & & $cog$ & $cycog$  & $\wadm$ & $wc_{dl}$ & $\strten$ & $ten$ & $ten_{sta}$  \\

        \hline

        $\F_1$& 
        \begin{tikzpicture}[thick,node distance =10mm, on grid, baseline=(current bounding box.center)]
                    \node[arg2] at (0,0) (a) {$a$};
                    \node[dot2] (b) [right of = a] {};

                    \draw[->] (b) to (a);
                    \draw[->] (b) to [in = 45, out = 315, looseness = 5] (b);
        \end{tikzpicture} & $a\in \cred_\sigma(\F_1)$? & Yes  & Yes &  Yes & Yes & Yes & Yes & Yes \\\hline
        
        
        $\F_2$ & \begin{tikzpicture}[thick,node distance =10mm, on grid, baseline=(current bounding box.center)]
                \node[arg2] at (0,0) (a) {$a$}; 
                \node[dot2] (b1) [above right = 5mm and  10mm of a] {};
                \node[dot2] (b2) [below  of= b1] {};
                \draw[->] (b1) to (a);
                \draw[->] (b2) to (a);
                \draw[<->] (b2) to (b1);
        \end{tikzpicture}  &$a\in \cred_\sigma(\F_2)$? & No & No & No & Yes & No & No & No \\\hline

        $\F_3$ & \begin{tikzpicture}[thick,node distance =10mm, on grid, baseline=(current bounding box.center)]
                \node[arg2] at (0,0) (a) {$a$}; 
                \node[dot2] (b1) [above right = 5mm and  10mm of a] {};
                \node[dot2] (b2) [below of= b1] {};
                \draw[->] (b1) to (a);
                \draw[->] (a) to (b2);
                \draw[->] (b2) to (b1);
        \end{tikzpicture} & $a\in \cred_\sigma(\F_3)$? & No & Yes & No& No & No & No & No \\\hline

        $\F_4$ & \begin{tikzpicture}[thick,node distance =10mm, on grid, baseline=(current bounding box.center)]
                \node[arg2] (a) at (0,0) {$a$};
                \node[dot2] (b)[right of=a] {}; 
                \node[dot2] (c) [above right = 5mm and  10mm of b] {};
                \node[dot2] (d) [below of= c] {};
                \draw[->] (b) to (a);
                \draw[->] (c) to (b);
                \draw[->] (d) to (c);
                \draw[->] (b) to (d);
        \end{tikzpicture} &$a\in \cred_\sigma(\F_4)$? & No & Yes & Yes & Yes & No & No & No \\\hline
        
         $\F_5$ &\begin{tikzpicture}[thick,node distance =10mm, on grid, baseline=(current bounding box.center)]
                \node[arg2] (a) at (0,0) {$a$};
                \node[dot2] (b)[right of=a] {}; 
                \node[dot2] (c) [above right = 5mm and  10mm of b] {};
                \node[dot2] (d) [below of= c] {};
                \draw[->] (b) to (a);
                \draw[->] (c) to (b);
                \draw[->] (d) to (c);
                \draw[->] (b) to (d);
                \draw[->] (c) to [in=45,out=175] (a);
                \draw[->] (d) to [in=315,out=185] (a);
        \end{tikzpicture} &$a\in \cred_\sigma(\F_5)$ & No & No & Yes & Yes & No & No & No \\\hline

        $\F_6$ & \begin{tikzpicture}[thick,node distance =10mm, on grid, baseline=(current bounding box.center)]
                    \node[arg2] (a) at (0,0) {$a$};
                    \node[dot2] (b) [above right = 5mm and  10mm of  a] {}; 
                    \node[dot2] (c) [below of = b] {};
                    \node[dot2] (d) [right of= b] {};
                    \node[dot2] (e) [below of= d] {};
                    \draw[->] (b) to (a);
                    \draw[->] (c) to (a);
                    \draw[->] (d) to (b);
                    \draw[->] (e) to (c);
                    \draw[<->] (b) to (c);
                    \draw[<->] (d) to (e);
            
        \end{tikzpicture} &$a\in \cred_\sigma(\F_6)$? & No & No & No & Yes & Yes & Yes & Yes \\\hline
        
        $\F_7$ & \begin{tikzpicture}[thick,node distance =10mm, on grid, baseline=(current bounding box.center)]
            \node[arg2] at (0,0) (a){$a$};
            \node[dot2]  (b1) [above left =10mm and 15mm of a]{};
                 \node[dot2] (b2)[above = 10mm  of a]{};
                 \node[dot2] (b3)[above right =10mm and 10mm of a]{};
    
                \node[dot2] (c3) [above of = b2] {};
                \node[dot2] (c2)[left of = c3]{};
                \node[dot2] (c1)[left of =c2]{};
                \node[dot2] (c4) [above of = b3] {};

                \draw[->] (b1) to (a);
                \draw[->] (b2) to (a);
                \draw[->] (b3) to (a);
                \draw[<->] (b2) to (b3);
                \draw[->] (c1) to (b1);
                \draw[->] (c2) to (b1);
                \draw[->] (c3) to (b2);
                \draw[->] (c4) to (b3);
                \draw[<->] (c1) to [out= 45, in =135] (c4);
                \draw[<->] (c2) to [out= 45, in =135] (c3);
        \end{tikzpicture} &$a\in \cred_\sigma(\F_{7})$? & No & No & No& Yes & No & Yes & Yes \\\hline
        
        $\F_8$ & \begin{tikzpicture}[thick,node distance =10mm, on grid, baseline=(current bounding box.center)]
            \node[arg2] at (0,0) (a){$a$};
            \node[dot2]  (b1) [above left =10mm and 15mm of a]{};
             \node[dot2] (b2)[above = 10mm  of a]{};
             \node[dot2] (b3)[above right =10mm and 10mm of a]{};
    
                \node[dot2] (c3) [above of = b2] {};
                \node[dot2] (c2)[left of = c3]{};
                \node[dot2] (c1)[left of =c2]{};
                \node[dot2] (c4) [above of = b3] {};

                \draw[->] (b1) to (a);
                \draw[->] (b2) to (a);
                \draw[->] (b3) to (a);
                \draw[<->] (b2) to (b3);
                \draw[->] (c1) to (b1);
                \draw[->] (c2) to (b1);
                \draw[->] (c3) to (b2);
                \draw[->] (c4) to (b3);
                \draw[<->] (c1) to [out= 45, in =135] (c4);
                \draw[<->] (c2) to [out= 45, in =135] (c3);
            
            \node[dot2] (d)[above left = 31.5mm and 5mm of a]{};
            \draw[<->] (d) to [out= 190, in=90](c1);
            \draw[<->] (d) to [out= 240, in=90](c2);
            \draw[<->] (d) to [out= 300, in=90](c3);
            \draw[<->] (d) to [out= 350, in=90](c4);    
        \end{tikzpicture} & $a\in \cred_\sigma(\F_{8})$? & No & No & No & Yes & No & No & Yes \\\hline
    \end{tabular}
    }
    \caption{In the leftmost columns are various argumentation frameworks with a selected argument $a$. On the right, a `Yes' indicates the argument $a$ in the corresponding AF is credulously accepted by the semantics in that column.}
    \label{tab: cred acceptance}
\end{table*}

We first show all the non-implications for the dashed arrow. That is, we show for each semantics $\sigma$ and $\tau$ so that there is no directed path in Figure \ref{fig: weak semantics summary-diagram} from $\sigma$ to $\tau$ that there exists a finite $\F$ and an extension $A\subseteq \F$ so that $A\in \sigma(\F)$, yet $A$ is not a subset of any $B\in \tau(\F)$. In fact, we do this by showing that there exists $\F$ so that $\cred_\sigma(\F)\not\subseteq \cred_\tau(\F)$. 

\begin{lemma}\label{lemma: minimal set of dashed counterexamples}
    For each of the following statements, there exists a finite AF $\F$ making the statement true:
    \begin{enumerate}[a)]
    \item $\cred_{wc_{dl}}(\F)\not\subseteq \cred_{\staten}(\F)$
    \item $\cred_{wc_{dl}}(\F)\not\subseteq \cred_{\wadm}(\F)$
    \item $\cred_{wc_{dl}}(\F)\not\subseteq \cred_{\cycog}(\F)$
    \item $\cred_{\wadm}(\F)\not\subseteq \cred_{\staten}(\F)$
    \item $\cred_{\wadm}(\F)\not\subseteq \cred_{\cycog}(\F)$
    \item $\cred_{\cycog}(\F)\not\subseteq \cred_{wc_{dl}}(\F)$
    \item $\cred_{\strten}(\F)\not\subseteq \cred_{\wadm}(\F)$
    \item $\cred_{\strten}(\F)\not\subseteq \cred_{\cycog}(\F)$
    \item $\cred_{\ten}(\F)\not\subseteq \cred_{\strten}(\F)$
    \item $\cred_{\staten}(\F)\not\subseteq \cred_{\ten}(\F)$
    \end{enumerate}
\end{lemma}
\begin{proof}
    See Table \ref{tab: cred acceptance} where the credulous acceptance of $a$ for various examples in each of these semantics has been computed. The following examples witness each of these non-inclusions:
    a)  $\F_2$,
    b)  $\F_2$,
    c)  $\F_2$,
    d)  $\F_4$,
    e)  $\F_5$,
    f)  $\F_3$,
    g)  $\F_6$,
    h)  $\F_6$,
    i)  $\F_7$,
    j)  $\F_8$.
\end{proof}

\begin{proposition}\label{prop: no dashed path}
    If there is no directed path from $\sigma$ to $\tau$ in Figure \ref{fig: weak semantics summary-diagram}, then there is a $A\subseteq \F$ so that $A\in \sigma(\F)$ and $A$ is not contained in any set $B\in \tau(\F)$. 
\end{proposition}
\begin{proof}
    Inspecting Figure \ref{fig: weak semantics summary-diagram}, we see that any $\sigma$ and $\tau$ so that there is no directed path (in arrows and dashed arrows) from $\sigma$ to $\tau$, there exists $\sigma',\tau'$ so that there is a directed path from $\sigma'$ to $\sigma$ and from $\tau$ to $\tau'$ so that Lemma \ref{lemma: minimal set of dashed counterexamples} gives an example of an argumentation framework $\F$ so that $\cred_{\sigma'}(\F)\not\subseteq \cred_{\tau'}(\F)$. Since there is a directed arrow from $\sigma'$ to $\sigma$, we have that $\cred_{\sigma'}(\F)\subseteq \cred_\sigma(\F)$. Similarly, since there is a directed path from $\tau$ to $\tau'$, we have that $\cred_{\tau}(\F)\subseteq \cred_{\tau'}(\F)$. Thus, if it were true that $\cred_{\sigma}(\F)\subseteq \cred_{\tau}(\F)$, then we could conclude that $\cred_{\sigma'}(\F)\subseteq \cred_{\sigma}(\F)\subseteq \cred_{\tau}(\F)\subseteq \cred_{\tau'}(\F)$, in contradiction to Lemma \ref{lemma: minimal set of dashed counterexamples}.

    Thus we conclude that $\cred_{\sigma}(\F)\not\subseteq \cred_{\tau}(\F)$. Let $a$ be in $\cred_{\sigma}(\F)\smallsetminus \cred_{\tau}(\F)$. Then there is some $A$ so that $a\in A$ and $A\in \sigma(\F)$. There cannot be any $B\supseteq A$ so that $B\in \tau(\F)$, since $a\notin \cred_\tau(\F)$.
\end{proof}

Now to show all the necessary non-implications in the solid line sense, i.e., $A\in \sigma(\F)\not\Rightarrow A\in \tau(\F)$, we need only consider pairs where there is a directed path involving a dashed line.

\begin{proposition}
    \label{prop: cogent not solid arrow wcomp}
    Cogent does not imply weakly complete. 
    \begin{proof}
        Consider the AF $\F$ with one element and no attacks. Then $\emptyset\in \cog(\F)$, but since the grounded extension is always contained in any weakly complete extension, $\emptyset$ is not weakly complete in $\F$.
    \end{proof}
\end{proposition}

\begin{proposition}\label{prop: no solid path}
    If there is no directed solid-line path from $\sigma$ to $\tau$ in Figure \ref{fig: weak semantics summary-diagram}, then there is some $A\subseteq \F$ with $\F$ finite so that $A\subseteq \sigma(\F)\smallsetminus \tau(\F)$.
\end{proposition}
\begin{proof}
    For the cases where there is no directed path from $\sigma$ to $\tau$ including dashed lines, Proposition \ref{prop: no dashed path} shows that there is an $A\subseteq \F$ with $\F$ finite so that $A\subseteq \sigma(\F)\smallsetminus \tau(\F)$. Thus we need only consider the cases where there is a directed path including dashed lines, but no directed path without including dashed lines. 
    
    Let $\sigma,\tau$ be so there is a directed path including dashed lines but no directed path without including dashed lines from $\sigma$ to $\tau$. Inspecting Figure \ref{fig: weak semantics summary-diagram}, we see that there is some $\sigma',\tau'$ so that there is a solid directed path from $\sigma'$ to $\sigma$ and from $\tau$ to $\tau'$ so that  
    Lemma \ref{prop: cogent not solid arrow wcomp} ensures that there is an $A\subseteq \F$ so that $A\in \sigma'(\F)\smallsetminus \tau'(\F)$. It follows that $A\in \sigma(\F)\smallsetminus \tau(\F)$.  
\end{proof}

\FloatBarrier
\end{document}